\documentclass{article}

\usepackage{microtype}
\usepackage{graphicx}
\usepackage{subfigure}
\usepackage{booktabs} 

\usepackage{hyperref}

\usepackage[utf8]{inputenc} 
\usepackage[T1]{fontenc}    
\usepackage{hyperref}       
\usepackage{url}            
\usepackage{booktabs}       
\usepackage{amsfonts}       
\usepackage{nicefrac}       
\usepackage{microtype}      
\usepackage{xcolor}         

\usepackage{amsmath}
\usepackage{natbib}
\usepackage{array}
\usepackage{makecell}
\usepackage{float}
\usepackage{graphicx}

\usepackage{blindtext}
\usepackage{soul}
\usepackage{bm}

\usepackage[accepted]{icml2025}

\usepackage{amsmath}
\usepackage{amssymb}
\usepackage{mathtools}
\usepackage{amsthm}

\usepackage[capitalize,noabbrev]{cleveref}

\theoremstyle{plain}
\newtheorem{theorem}{Theorem}[section]

\newtheorem{lemma}[theorem]{Lemma}
\newtheorem{corollary}[theorem]{Corollary}
\theoremstyle{definition}
\newtheorem{definition}[theorem]{Definition}
\newtheorem{assumption}[theorem]{Assumption}
\theoremstyle{remark}
\newtheorem{remark}[theorem]{Remark}

\allowdisplaybreaks

\newcommand{\norm}[1]{\left\lVert#1\right\rVert}
\newcommand{\vertiii}[1]{{\left\vert\kern-0.25ex\left\vert\kern-0.25ex\left\vert
    #1 \right\vert\kern-0.25ex\right\vert\kern-0.25ex\right\vert}}

\newcommand{\x}{\mathbf{x}}
\newcommand{\y}{\mathbf{y}}
\newcommand{\z}{\mathbf{z}}
\newcommand{\uu}{\mathbf{u}}
\newcommand{\bfb}{\mathbf{b}}
\newcommand{\vv}{\mathbf{v}}

\newcommand{\nout}{n_{\text{\normalfont out}}}
\newcommand{\nin}{n_{\text{\normalfont in}}}
\newcommand{\dec}{\text{\normalfont Dec}}
\newcommand{\enc}{\text{\normalfont Enc}}
\newcommand{\pool}{\text{\normalfont Pool}}
\newcommand{\bb}{\text{\normalfont B}}

\usepackage[textsize=tiny]{todonotes}

\icmltitlerunning{Length independent generalization bounds for deep SSM
architectures via Rademacher contraction}

\begin{document}

\twocolumn[
\icmltitle{Length independent generalization bounds for deep SSM architectures via Rademacher contraction and stability constraints}




\begin{icmlauthorlist}
\icmlauthor{Dániel Rácz}{sztaki,elte}
\icmlauthor{Mihály Petreczky}{lille}
\icmlauthor{Bálint Daróczy}{sztaki}
\end{icmlauthorlist}

\icmlaffiliation{sztaki}{HUN-REN SZTAKI, Budapest, Hungary}
\icmlaffiliation{elte}{ELTE, Budapest, Hungary}
\icmlaffiliation{lille}{Univ. Lille, CNRS, Centrale Lille, UMR 9189 CRIStAL}

\icmlcorrespondingauthor{Dániel Rácz}{racz.daniel@sztaki.hu}

\icmlkeywords{generalization bounds, state-space models, stability, Rademacher}

\vskip 0.3in
]



\printAffiliationsAndNotice{}  

\begin{abstract}
Many state-of-the-art models trained on long-range sequences, for example S4,
  S5 or LRU, are made of sequential blocks combining State-Space Models (SSMs)
  with neural networks. In this paper we provide a PAC bound that holds for
  these kind of architectures with \emph{stable} SSM blocks and does not depend
  on the length of the input sequence. Imposing stability of the SSM blocks is a
  standard practice in the literature, and it is known to help performance. Our
  results provide a theoretical justification for the use of stable SSM blocks
  as the proposed PAC bound decreases as the degree of stability of the SSM
  blocks increases.
\end{abstract}

Preliminary version of this work has been presented as a workshop paper 'Length independent generalization bounds for deep SSM architectures' at the Next Generation of Sequence Modeling workshop at ICML 2024.

\section{Introduction}


The challenge of learning rich representations for long-range sequences (time series, text, video) has persisted for decades. RNNs, including LSTMs \cite{hochreiter1997long} and GRUs \cite{cho2014properties}, struggled with long-term dependencies, while Transformers, despite improvements, still perform poorly on difficult tasks \cite{huang2024advancingtransformerarchitecturelongcontext,amos2024trainscratchfaircomparison}.

Recently, several novel architectures have been proposed which outperform
previous models by a significant margin, for an overview
see \cite{huang2024advancingtransformerarchitecturelongcontext,amos2024trainscratchfaircomparison}.
One notable class of such architectures are
the so-called deep State-Space Models (deep SSMs),
which typically contain several layers made of the composition of
dynamical systems of either continuous or discrete time, and non-linear
transformations (e.g. MLP) \cite{gu2021combining,gu2023how,mamba,wang2024state,S4,S4D,S5,H3,LRU}.
While SSM architectures have been extensively validated empirically,
the theoretical foundations of SSMs are
less understood.  One key point of these
models is that they are - often implicitly - equipped with some form of stability constraints for the SSM components.
This motivates the question:
\vspace{-\topsep}
\begin{center}
    \textit{What is the role of stability in the success of deep SSM
     architectures for long-range sequences?}
\end{center}
\vspace{-\topsep}
We partially address this problem by leveraging stability to derive a PAC bound which is independent of input sequence length.
%
Our contributions are:

\textbf{Norms of SSMs for bounding Rademacher complexity.}
We show that the Rademacher complexity of SSMs can be upper bounded by their system norms, such as the $H_2$ and $\ell_1$ norms \cite{chellaboina1999induced}, which are well-known in control theory and linked to quadratic stability. This highlights stability not just as a practical necessity, but as a fundamental aspect of SSM architectures, making it the key takeaway of our work.

\textbf{Rademacher Contractions.}
 We upper bound the Rademacher complexity of multilayer deep SSM models, encompassing many popular architectures, by introducing the concept of \emph{Rademacher Contraction} (RC), which, similarly to the celebrated Talagrand's Contraction Lemma \cite{ledoux1991probability}, allows us to bound the Rademacher complexity of deep models. 

\textbf{PAC bound for deep SSMs.} Using the concept of \emph{Rademacher Contraction} we establish a PAC-bound on the
generalization error of deep SSMs. The resulting bound is independent of the input sequence length and depends only implicitly on the depth of the model.
%
Our results cover both classification and regression tasks for most of the popular deep SSM architectures.

\textbf{Outline of the paper.} In Section~\ref{sec:related} we present the
related literature, then we set some notation and present an informal statement of our result in Section \ref{sect:informal}. The formal problem statement along with the remaining notation and our assumptions
are in Section~\ref{sec:problem}. We propose our main result and a sketch of the proof in
Section~\ref{sec:main}. A numerical example illustrating the result is in Section \ref{sec:experiment}.
The
majority of the proofs and some additional details are shown in the Appendix.

\section{Related work}
\label{sec:related}
Apart from \cite{liu2024from,SSMComparableExpressive}, SSM research primarily addresses modeling power, parametrization, and computational complexity, with limited focus on generalization bounds.

\textbf{Theoretical analysis of SSMs.}
SSM modeling power has been studied via approximation capabilities \cite{ExpressiveSSMRoughPath,ExpressivitySSMStable,ExpressivitySSMComplex,ExpressiveSSMRoughPath}, with a survey in \cite{tiezzi2024resurgencerecurrentmodelslong}. This paper, however, focuses on statistical generalization bounds. \cite{SSMComparableExpressive} derives statistical bounds on SSM approximation error, but only for a specific learning algorithm and parametrization, whereas our PAC bound is algorithm-agnostic. Experimental results suggest stable SSM parametrizations improve learning \cite{ExpressivitySSMStable,parnichkun2024statefreeinferencestatespacemodels,S4,S5,H3,mamba,SSMLearningTheory}. Computational complexity of inference and learning has been analyzed in \cite{massaroli2023laughing,gu2021combining,gu2023how}, while \cite{gu2023how,SSMFrequencyTuning,ExpressivitySSMStable} investigated initialization techniques.

\textbf{PAC bounds for single-layer SSMs.}
\cite{liu2024from} derived a PAC bound for a single continuous-time LTI SSM using Rademacher complexity. In contrast, their result applies only to a single SSM block without nonlinear elements, their bound grows with sequence length, and it does not account for discretization effects. Moreover, their constants are not directly linked to control-theoretic quantities like $H_2$ norms. Their input assumptions also differ: while we assume bounded $\ell_2$ norm inputs, they allow unbounded inputs, but require subgaussianity and continuity, the latter being inapplicable in discrete time and potentially imposing constraints on the sampling mechanism.


\textbf{PAC bounds for RNNs.}
Prior PAC bounds for RNNs use VC-dimension or covering numbers \cite{KOIRAN199863,sontag1998learning,pmlr-v144-hanson21a}, Rademacher complexity \cite{WeiRNN,Akpinar_Kratzwald_Feuerriegel_2020,pmlr-v161-joukovsky21a, pmlr-v108-chen20d, tu2020understanding}, or PAC-Bayesian methods \cite{pmlr-v80-zhang18g}. Since LTIs are core components of SSMs and a subclass of RNNs, RNN generalization bounds are relevant. However, prior PAC bounds grow with integration time (continuous) or time steps (discrete), limiting their use for long-range sequences. The bound for single vanilla RNNs in \cite{pmlr-v108-chen20d} upper-bounds the $\ell_1$ norm used here. In contrast, our bound is independent of state-space dimension and assumes only bounded $\ell_1$/$H_2$ norms.

\textbf{PAC bounds for Neural Ordinary Differential Equations.}
PAC bounds for NODEs have been developed in \cite{hanson2021learning,hanson2024rademacher,marion2023generalization,fermanian2021framing}. These results are based on Rademacher complexity and they are either affine in inputs or defined in the rough path sense. While a single block SSM interpreted in continuous time is affine in the input, general multi-block SSMs do not fall into this category. Moreover, these bounds are still exponential in the length of the
integration interval, i.e., the length of the time series if fixed sampling time is used.


\textbf{PAC bounds for deep networks and transformers.}
\cite{trauger2024sequence} derive a sequence-length-independent Rademacher complexity bound for single-layer transformers, improving slightly on \cite{edelman2022inductive} for multi-layer cases, though their bound grows logarithmically with the sequence length.
\\
However, their results do not apply to SSMs and involve matrix norms that may scale with the attention matrix size. Maintaining norm stability for longer sequences requires reducing certain matrix entries. In contrast, the $H_2/\ell_1$ norms in this paper depend only on state-space matrices and remain invariant to the input length.
\\
Generalization bounds for deep neural networks (DNNs) extend beyond RNNs and dynamical systems \cite{bartlett2017spectrally, liang2019fisher, golowich2018size, truong2022rademacher}. Since deep SSMs can simulate DNNs and resemble feedforward networks with fixed input length, their bounds should align with DNN results. \cite{golowich2018size} provide a depth-independent bound under bounded Schatten p-norm and a polynomial-depth bound for ReLU networks via contraction. Other works use spectral \cite{bartlett2017spectrally} or Fisher-Rao norms \cite{liang2019fisher} to mitigate depth dependence. \cite{truong2022rademacher} further refine \cite{golowich2018size} to a depth-independent, non-vacuous bound for non-ReLU activations.
When applied to deep SSMs with trivial state-space components, the bounds of the present paper are more conservative than those of \cite{golowich2018size} for general activation functions, but are consistent with  \cite{golowich2018size,truong2022rademacher} for ReLU activation functions.

For further discussion of the literature in the fields of PAC-Bayesian, finite-sample bounds and non i.i.d. data, see Appendix~\ref{app:related}.

\section{Informal statement of the result}
\label{sect:informal}

\textbf{Learning problem.}
We consider the usual supervised learning framework for sequential input data.
That is, we consider a family $\mathcal{F}$ of models, each model $f \in \mathcal{F}$
is a function which maps sequences of elements
$\uu[1],\ldots, \uu[T]$ of the \emph{input space}
$\mathbb{R}^{\nin}$ to outputs (labels) in $\mathcal{Y} \subseteq
\mathbb{R}^{\nout}$. We fix the length of the sequences to $T$.
A dataset is an i.i.d sample of the form $S =
\{(\uu_{i}, \y_{i})\}_{i=1}^{N}$ from
some probability distribution $\mathcal{D}$, where
$\uu_i$ is a sequence of length $T$ having elements in $\mathbb{R}^{\nin}$, and $\y_i$ belongs to $\mathcal{Y}$.
The probability measure determined by $\mathcal{D}$
is defined on a $\sigma$-algebra generated by the Borel sets of  $\mathbb{R}^{\nin T} \times \mathcal{Y}$,
where the set of sequences of elements of $\mathbb{R}^{\nin}$ of length $T$ is identified with
$\mathbb{R}^{\nin T}$.
We use the symbols
$E_{(u,y) \sim \mathcal{D}}$, $P_{(u,y) \sim \mathcal{D}}$,
$E_{S \sim \mathcal{D}^N}$ and $P_{S \sim \mathcal{D}^N}$
to denote expectations and probabilities w.r.t. a probability measure $\mathcal{D}$
and its $N$-fold product $\mathcal{D}^N$ respectively.
The notation $S \sim \mathcal{D}^N$
tacitly assumes that $S \in (\mathcal{U} \times \mathcal{Y})^N$, i.e.
$S$ is made of $N$ tuples of input and output trajectories.
%

An elementwise loss function is a function $\ell: \mathcal{Y} \times
\mathcal{Y} \to \mathbb{R^+}$ such that $\ell(y,y')=0$ iff $y=y'$.
Its role is to measure the discrepancy between predicted and true outputs (labels).

We define the \emph{empirical loss} as
\( \mathcal{L}_{emp}^{S}(f)
= \frac{1}{N}\sum_{i=1}^{N} \ell(f(\uu_{i}),y_{i}) \)
and the \emph{true loss} as
\( \mathcal{L}(f) =
\mathbf{E}_{(\uu,y)\sim \mathcal{D}}[\ell(f(\uu), y)] \).
The goal of this paper is to bound the \emph{generalization error
or gap}, $\mathcal{L}(f)-\mathcal{L}_{emp}^{S}(f)$ uniformly for all models $f \in \mathcal{F}$. 

We will be interested in model classes $\mathcal{F}$, elements of which arise by combining neural networks and the so-called State-Space Models.

 \textbf{Model class of SSMs.}
A \emph{State-Space Model (SSM)} is a discrete-time linear dynamical system of
the form
\begin{equation}
    \label{eq:dtlti}
    \Sigma
    \begin{cases}
      &\x[k+1] = A \x[k] + B\uu[k], ~ \x[1] = 0 \\
      &\y[k]=C \x[k] + D \uu[k]
    \end{cases}
\end{equation}
where $A \in \mathbb{R}^{n_{x} \times n_{x}}, B \in \mathbb{R}^{n_{x} \times
n_u}, C \in \mathbb{R}^{n_y \times n_{x}}$ and $D \in \mathbb{R}^{n_y
\times m}$ are matrices,
$\uu[k],\x[k]$ and $\y[k]$ are the input, the state and the output signals respectively for $k=1,2,
\ldots, T$, where $T$ is the number of time steps. 
In this paper we are interested in stable SSMs, i.e., in SSMs \eqref{eq:dtlti} for which the matrix $A$ is Schur (the eigenvalues are inside the complex unit disk).
Intuitively, stable SSMs are robust to perturbations, i.e., their state and output are continuous in the initial state and input, see for instance \cite{AntoulasBook} for more details.
\begin{remark}[Selective SSMs]
For simplicity, we exclude the so called \emph{selective SSM} models, where $A$ and $B$ depend on $\uu[k]$ at each step (see \eqref{ct:ssme:eq2} in Remark \ref{rem:cont-time}, Appendix \ref{appendix_ssm} for a precise definition). While they are more general, than LTI models and widely used in deep SSMs \cite{mamba}, they present greater analytical challenges. Extending our results to such models remains future work.
\end{remark}

%
An SSM \eqref{eq:dtlti}
induces a linear function $\mathcal{S}_{\Sigma,T}$ which maps every
input sequence $\uu[1],\ldots, \uu[T]$ to the output sequence
$\y[1],\ldots,\y[T]$.
In particular, $\mathcal{S}_{\Sigma,T}$ has a
well-defined induced norm as a linear operator, defined in the usual way.
For stable SSMs this norm can be bounded uniformly in $T$.
\\
A \emph{SSM block} is a residual composition of the SSM with a non-linear function $g$ applied element-wise, i.e. an SSM block maps the sequence
$\uu[1],\ldots, \uu[T]$ to the sequence
defined by $f^{\text{DTB}}(\uu)[k] =
g(\mathcal{S}_{\Sigma,T}(\uu)[k]) + \alpha \uu[k]$, where $\alpha \in \mathbb{R}$ is the residual weight.
A \emph{deep
SSM model} is a composition of several SSM blocks with an encoder, and a decoder
transformation preceded by a time-pooling layer. That is, the input-output
map of a deep SSM is a composition of functions of the form
$f^{\dec} \circ f^{\pool} \circ f^{\bb_L} \circ \ldots \circ f^{\bb_1} \circ f^{\enc}$, where $\circ$ denotes composition of functions.  The functions $f^{\enc}$ and $f^{\dec}$ are linear transformations
 which are constant in time and are applied to sequences element-wise, while $f^{\bb_i}$ is the
 input-output map of an SSM block for all $i$. This definition covers many examples from the literature, e.g. S4 \cite{S4}, S4D \cite{S4D}, S5 \cite{S5} or LRU \cite{LRU}.

The main result of this paper is the following \textbf{PAC bound for deep SSMs}:

 \begin{theorem}[Informal theorem]
   \label{thm:maingeneral_informal}
   Let $\mathcal{F}$ be a set of deep SSM models with stable SSM blocks, which satisfy a number of mild regularity assumptions.
  There exist constants $K_l$, and $K_{\mathcal{F}}$ which depend only on the model class
  $\mathcal{F}$, such that
  for any time horizon $T > 0$, any confidence level $\delta > 0$, with probability at least $1-\delta$ over data samples $S \sim \mathcal{D}^N$,
   \begin{align}
       \label{eq:main_info}
           \forall f \in \mathcal{F}: \mathcal{L}(f)          \! -\! \mathcal{L}_{emp}^S(f)
           \leq \frac{K_{\mathcal{F}} + K_l \sqrt{2\log(\frac{4}{\delta})}}{\sqrt{N}}
   \end{align}
\end{theorem}
With standard arguments on PAC bounds and Rademacher complexity, the result above also implies the following oracle inequality for the Empirical Risk Minimization framework \cite{shalev2014understanding}.
\begin{corollary}
\label{thm:maingeneral_informal:col}
With the assumptions of Theorem \ref{thm:maingeneral_informal}
for $f_{ERM}=\mathrm{arg min}_{f \in \mathcal{F}} \mathcal{L}^{\mathcal{S}}_{emp}(f)$, for any
$\delta > 0$,  with probability
at least $1-\delta$ over  data samples $S \sim \mathcal{D}^N$:
    \begin{align}
       \label{eq:main_info1}
           \mathcal{L}(f_{\text{\normalfont ERM}})
           \leq \min_{f \in \mathcal{F}} \mathcal{L}(f)+
           \frac{K_{\mathcal{F}}+5K_l\sqrt{2\log(8/ \delta)}}{\sqrt{N}}
   \end{align}
\end{corollary}
Bounds \eqref{eq:main_info} and \eqref{eq:main_info1} ensure that as $N$ grows, the empirical and true losses converge, and the learned model's true loss approaches the minimum possible loss.
\\
The term $K_{\mathcal{F}}$ depends on the norms of the SSM blocks and the magnitudes of non-SSM weights, but
it remains independent of $T$.
Since in general, norms of SSMs decrease as their stability increases,
\emph{stability makes the generalization gap insensitive to sequence length, and increasing stability further decreases it.}
\\
Specifically, for deep SSMs with $k$ layers, $K_{\mathcal{F}} = O((\text{SSM norm})^k (\text{non-SSM weight norm})^k)$. While $K_{\mathcal{F}}$ grows exponentially with the depth unless all components are contractions, high non-SSM weights can be offset by lower SSM norms. These norms decrease as SSMs become more stable, though stability is not directly tied to weight magnitudes — stable SSMs can still have large weights. This exponential dependence aligns with bounds for  deep neural network \cite{golowich2020size,truong2022generalization}.
\\
Depth may negatively impact the generalization gap, but this does not imply poor generalization overall. 
Even if $K_{\mathcal{F}}$ is large for deep SSMs,
\eqref{eq:main_info1} implies that if the best true error is small 
then the generalization gap can still be small.
Additionally, as $N$ increases, the influence of $K_{\mathcal{F}}$ diminishes, suggesting deeper models require more data, which is consistent with findings on deep neural networks.

\section{Formal problem setup}
\label{sec:problem}

\textbf{Notation.} We denote scalars with lowercase characters, vectors with
lowercase bold characters and matrices with uppercase characters. For a matrix
$A$ let $A_i$ denote its $i$th row. The symbol $\odot$ denotes the elementwise
product.  We use $[n]$ to denote the set $\{1,2,\ldots,n\}$ for $n \in
\mathbb{N}$.
For vector valued time function $\uu$, the notation 
$\uu^{(j)}[t]$
refers to the $j$th
coordinate of the value of  function at time $t$.
Furthermore, we use
$\Sigma$  to denote a dynamical system specified in
the context. The constant $\nin$ refers to the dimension of the input sequence,
$T$ refers to its length in time, while $\nout$ is the dimension of the output
(not necessarily a sequence).
\\
Denote by $\ell^{2}_T(\mathbb{R}^n)$ and  $\ell^{\infty}_T(\mathbb{R}^n)$ the finite-dimensional Banach spaces generated by the all finite
sequences over $\mathbb{R}^n$ of length $T$, viewed as vectors of $\mathbb{R}^{nT}$, with the Eucledian norm $\|\cdot\|$ and
the supremum norm $\|\cdot\|_{\infty}$
over $\mathbb{R}^{nT}$ respectively.
If $\mathcal{X}$ is a Banach space, we denote its
norm by $\|\cdot\|_{\mathcal{X}}$.
In particular, $\norm{\uu}^2_{\ell^{2}_T(\mathbb{R}^n)} = \sum_{k=1}^{T}
\norm{\uu[k]}_2^2$, and
$\norm{\uu}_{\ell^{\infty}_T(\mathbb{R}^n)}=\sup_{k \in [T],j \in [n]} |\uu^j[k]|$.
For a Banach space
$\mathcal{X}$, $B_{\mathcal{X}}(r)=\{ x \in \mathcal{X} \mid
\norm{x}_{\mathcal{X}} \le r\}$ denotes the ball of radius $r>0$ centered at
zero.

\subsection{Deep SSMs}
\label{subsec:deepssm}
\textbf{Stable SSMs.}
In the sequel, we consider
solutions of \eqref{eq:dtlti}
on the time interval
$[1,T]$, where
the value of $T$ is fixed.
As it was mentioned in Section \ref{sect:informal}, 
LTI systems
\eqref{eq:dtlti} 
are \emph{internally stable}. It is well-known \cite{AntoulasBook} that
internal stability is equivalent to the $A$ matrix in \eqref{eq:dtlti} is Schur, i.e.,
meaning all the eigenvalues of $A$ are inside the complex unit disk.
In particular, a sufficient (but not necessary) condition for stability is that
$A$ is a contraction, i.e. $\|A\|_2 < 1$.
\\
All popular architectures use stable SSM blocks, see Remark \ref{stab} in Appendix \ref{appendix_ssm}.
SSMs are often derived by discretizing continuous-time linear differential equations. With a fixed step $\Delta$, this yields a time-invariant linear (LTI) system \eqref{eq:dtlti}, see
Remark \ref{rem:cont-time}, Appendix \ref{appendix_ssm} for details.

\textbf{Input-output maps of SSMs as operators on $\ell^{p}_T$,$p=\infty,2$.}
As it was mentioned in Section \ref{sect:informal}, SSM \eqref{eq:dtlti}
induces an input-output map $\mathcal{S}_{\Sigma,T}$, which maps every
input sequence $\uu[1],\ldots, \uu[T]$ to output sequence $\y[1],\ldots,\y[T]$, and can be described by a convolution
$\y[t]=\mathcal{S}_{\Sigma}(\uu)[t]=\sum_{j=1}^{t} H_{j-1} \uu[t-j+1]$,
where $H_0=D$ and $H_j=CA^{j-1}B$, $j > 0$.
The map $\mathcal{S}_{\Sigma,T}$ can be
viewed as a  linear operator $\mathcal{S}_{\Sigma,T}:\ell^{p}_T(\mathbb{R}^{n_{u}}) \to \ell^{\infty}_T(\mathbb{R}^{n_{y}})$,
for any choice $p \in \{\infty,2\}$.
In particular, $\mathcal{S}_{\Sigma,T}$ has a
well-defined induced norm as a linear operator, defined in the usual way,
\[ \norm{\mathcal{S}_{\Sigma,T}}_{\infty,p}=\sup_{\uu \in
\ell^{p}_T(\mathbb{R}^{n_u})}\frac{\norm{\mathcal{S}_{\Sigma, T}(\uu)}
_{\ell^{\infty}_T(\mathbb{R}^{n_y})}}{\norm{\uu}_{\ell^{p}_T(\mathbb{R}^{n_y})}}.
\]
It is a standard result in control theory that if $\Sigma$ is internally stable, the
supremum
$\norm{\Sigma}_{\infty,p} = \sup_{T > 0} \norm{\mathcal{S}_{\Sigma,T}}_{\infty,p}$ of these
norms is finite, see \cite{AntoulasBook}.
In this paper, we will use upper bounds on
the induced norms
$\|\mathcal{S}_{\Sigma,T}\|_{r,\infty}$,
$r \in \{2,\infty\}$
to bound the Rademacher complexity. In
turn,  these norms can be upper bounded by the following two standard
control-theoretical norms defined on SSMs. 
%
%
For an SSM $\Sigma$ of the form \eqref{eq:dtlti} let us define the
$\ell_1$  \cite{chellaboina1999induced} and $H_2$ \cite{AntoulasBook} norms of $\Sigma$, denoted by $\norm{\Sigma}_1$
and $\norm{\Sigma}_2$ respectively, as
\begin{align*}
& \norm{\Sigma}_1 := \max\limits_{1 \leq i \leq n_y}\left[
    \norm{D_i}_1 + \sum\limits_{k=0}^{\infty} \norm{C_i A^k B}_1\right], \quad \\
&    \norm{\Sigma}_2 := \sqrt{
        \norm{D}_F^2 + \sum\limits_{k=0}^{\infty} \norm{C A^k B}_F^2.
    }
\end{align*}
Norms $\norm{\Sigma}_i$,$i=1,2$ can be easily computed, see Remark \ref{comp:norm}, Appendix \ref{appendix_ssm}.
The following Lemma connects $\mathcal{S}_{\Sigma,T}$ to the $\ell_1$ and $H_2$ norms.
\begin{lemma}[\cite{chellaboina1999induced}]
\label{lem:normest}
    For a system of form \eqref{eq:dtlti},
    it holds
    that
    $\sup_{T \ge 0} \norm{\mathcal{S}_{\Sigma,T}}_{2,\infty} \leq \norm{\Sigma}_1$ and
     $\sup_{T \ge 0} \norm{\mathcal{S}_{\Sigma,T}}_{\infty,\infty} \leq \norm{\Sigma}_2$.
\end{lemma}

\textbf{Deep SSM models.} In this paper, we consider deep SSM 
models, which consist of layers of blocks, each
block representing an SSM
followed by
a nonlinear transformation (MLP, GLU). Moreover,
these blocks are preceded by a linear encoder
and succeeded by a pooling block and a linear decoder.
\\
In order to define deep SSMs, first we define
MLP and GLU  layers. Then we define SSM blocks, which are compositions of SSMs \eqref{eq:dtlti}
with MLP and GLU layers. Finally, we define deep SSM models, where all these elements are combined.
\begin{definition}[MLP layer]
    \label{def:mlp}
    An MLP layer is a function
    $f:\ell_T^{\infty}(\mathbb{R}^{n_y}) \rightarrow \ell_T^{\infty}(\mathbb{R}^{n_u})$ such that there exist
    an integer $M \ge 1$,
    matrices and vectors $\{W_i,\bfb_i\}_{i=1}^{M}$
    and activation function $\rho: \mathbb{R} \rightarrow \mathbb{R}$, such
    that $W_i \in
    \mathbb{R}^{n_{i+1} \times n_{i}}$ and $\bfb \in \mathbb{R}^{n_i}$, $i \in [M]$,
    $n_1 = n_y$ and $n_{M+1} = n_u$, and
    \begin{equation}
    \label{def:mlp:eq1}
       f(\uu)[k]\!=\!g_{W_{M+1},
    \bfb_{M+1}} \circ  f_{W_M,\bfb_M} \circ \ldots \circ f_{W_{1}, \bfb_{1}} (\uu[k])
    \end{equation}
    where $k \in [T]$,
    $f_{W_i,\bfb_i}(x)=\rho(g_{W_i,\bfb_i}(\mathbf{x}))$ and $g_{W_i,\bfb_i}(\mathbf{x})=W_i\mathbf{x}+\bfb_i$ for all $i \in [M+1]$.
    By slightly abusing the notation, for a vector $\mathbf{v}$, $\rho(\mathbf{v})$ denotes the elementwise application of $\rho$ to $\mathbf{v}$.
\end{definition}
Intuitively, a MLP layer represents a deep neural network applied to a signal at every time step. The function $f_{W_i,\bfb_i}$ represents the $i$th layer of this neural network, with activation function $\rho$ and weights $W_i,\bfb_i$. For the sake of simplicity, activation function is assumed to be the same across all layers of the neural network.



\begin{definition}[GLU layer \cite{S5}]
    \label{def:glu}
    A GLU layer is a function of the form $f: \ell_T^{\infty}(\mathbb{R}^{n_y}) \to \ell_T^{\infty}(\mathbb{R}^{n_u})$
    parametrized by a matrix $W$ such that
    \begin{equation}
    \label{def:glu:eq1}
    f(\uu)[k] =
    GELU(\uu[k]) \odot \sigma(W \cdot GELU(\uu[k])),
    \end{equation}
    where $\sigma$ is the sigmoid
    function and GELU is the Gaussian Error Linear Unit
    \cite{hendrycks2016gaussian}.
\end{definition}

Note that this definition of GLU layer differs from the original definition in
\cite{dauphin2017language}, because in deep SSM models GLU is usually applied
individually for each time step, without any time-mixing operations. See
Appendix G.1 in \cite{S5}.

Next, we define a SSM block, which
is a composition of an SSM layer with a MLP/GLU layer.
\begin{definition}
    \label{def:ssmblock}
An SSM block is a function $f^{\text{DTB}}:
\ell_T^{r}(\mathbb{R}^{n_u}) \to \ell_T^{\infty}(\mathbb{R}^{n_u})$, $r \in \{2,\infty\}$,
such that for all $k \in [T]$
\begin{equation}
\label{def:ssmblock:eq1}
f^{\text{DTB}}(\uu)[k] =
g \circ \mathcal{S}_{\Sigma,T}(\uu)[k]) + \alpha \uu[k]
\end{equation}
for some SSM $\Sigma$ \eqref{eq:dtlti},
some MLP or GLU layer $g: \ell_T^{\infty}(\mathbb{R}^{n_y}) \to \ell_T^{\infty}(\mathbb{R}^{n_u})$
and constant
$\alpha$.
\end{definition}

We incorporate $\alpha$ so that the definition covers residual connections
(typically $\alpha$ is either 1 or 0). The
definition above is inspired by
the series of popular architectures mentioned in the introduction.

Finally, following the literature
on SSMs, we define
a deep
SSM model as a composition of SSM blocks along with linear layers (encoder/decoder) 
combined with a time-pooling layer in case of classification.
\begin{definition}[encoder, decoder, pooling]
\label{lin_layer}
An encoder is a function
$f: \ell_T^{p}(\mathbb{R}^{\nin}) \to \ell_T^{p}(\mathbb{R}^{n_u})$,
where $p \in \{2,\infty\}$ is an integer,
such that there exists a matrix $W_{\enc}$ for which
$f(\uu)[k] = W_{\enc} \uu[k]$.
A decoder is a function $f: \mathbb{R}^{n_u} \to \mathbb{R}^{\nout}$ such that there exists a matrix
$W_{\dec}$ such that $f(\mathbf{x})=W_{\dec}\mathbf{x}$.
A pooling layer is the function
$f^{\pool}: \ell_T^{\infty}(\mathbb{R}^{n_u}) \to \mathbb{R}^{n_u})$
defined by
$f^{\pool}(\uu) = \frac{1}{T}
   \sum_{k=1}^{T} \uu[k]$.
\end{definition}
An encoder corresponds to applying linear transformations to each element of the input sequence. The pooling layer is
typically an average pooling over the time axis.
\begin{definition}
    \label{def:dtdeepssm}
   A deep SSM model
   is a function $f:
   \ell_T^{2}(\mathbb{R}^{\nin}) \to \mathbb{R}^{\nout}$
   of the form
   \begin{equation}
    \label{def:dtdeepssm:eq1}
   f =
   f^{\dec} \circ f^{\pool} \circ f^{\bb_L} \circ \ldots \circ f^{\bb_1} \circ
   f^{\enc}
   \end{equation}
   where $f^{\enc}$ is an encoder and $f^{\dec}$ is a decoder, and
   $f^{\bb_i}$ are SSM blocks for all $i$, and
   $f^{\pool}$ is the pooling layer.
\end{definition}
Definition \ref{def:dtdeepssm} covers many important architectures from the literature, e.g. S4, S4D, S5 and LRU.
Note that we did not include such commonly used normalization techniques as batch normalization in the definition since they are not relevant for our results.
Indeed,
once the model training is finished, a batch normalization layer corresponds to applying a
neural network with linear activation function, i.e., it can be integrated into
one of the neural network layers. Since the objective of PAC bounds is to bound
the generalization error for already trained models, for the purposes of PAC
bounds, normalization layers can be viewed as an additional layer of neural network.

\subsection{Assumptions}

Before moving forward to discuss the main result, we summarize the assumptions
we make in the paper for the sake of readability. 


\begin{assumption}
    \label{ass:all}
We consider a family $\mathcal{F}$ of deep SSM models
of depth $L$ such that the following hold:

\textbf{1. Architecture.}
There exist families
$\mathcal{F}_{\text{Enc}}$ of encoders,
$\mathcal{F}_{\text{Dec}}$ of decoders,
$\mathcal{E}$ of SSMs,
$\mathcal{F}_{i}$, $i \in [L]$, of nonlinear blocks,
and collection of residual weights $\{\alpha_i\}_{i=1}^{L}$
such that if $f \in \mathcal{F}$ of the form
\eqref{def:dtdeepssm:eq1}, then
\\
\textbf{(1)}
the encoder $f^{\enc}$ belongs to $\mathcal{F}_{\enc}$, the decoder $f^{\dec}$ belongs to $\mathcal{F}_{\text{Dec}}$, \\
\textbf{(2)}
and if the $i$th SSM block $f^{\bb_i}$ is of the form
    \eqref{def:ssmblock:eq1},
    then $\Sigma \in \mathcal{E}$, $\alpha=\alpha_i$, and
$g$  belongs to $\mathcal{F}_{i}$.

\textbf{2. Scalar output.} Let $\nout = 1$.

\textbf{3. Lipschitz loss function.} Let the elementwise loss $\ell$
        be $L_\ell$-Lipschitz continuous, i.e.,
$\ell(y_1,y_1')-\ell(y_2,y_2') \le L_{\ell}(|y_1-y_2|+|y_1'-y_2'|)$ for all $y_1,y_2,y_1',y_2' \in \mathbb{R}$.

\textbf{4. Bounded input.} There exist $K_\uu > 0$ and $K_y > 0$ such
        that for any input trajectory $\uu$ and label $y$ sampled from
        $\mathcal{D}$, with probability 1 we have that
        $\norm{\uu}_{\ell_T^{2}(\mathbb{R}^{\nin})} \leq K_\uu$
        and  $|y| \leq
        K_y$.

\textbf{5. Stability \& bounded norms.}
         There exist constants $K_1$ and $K_2$
         such that $\|\Sigma\|_p \le K_p$, $p=1,2$ for each internally stable $\Sigma \in \mathcal{E}$.
        \item \textbf{Bounded encoder and decoder.}
        There exists constants $K_{\enc},K_{\dec}$
        such that if $f \in \mathcal{F}_{\text{Enc}}$,
        and $f(\uu)[k]=W_{\enc}\uu[k]$ for a matrix $W_{\enc}$, then
        $\norm{W_{\enc}}_{2, 2} < K_{\enc}$, and
        if $f \in \mathcal{F}_{\dec}$ and
        $f(\mathbf{x})=W_{\dec}\mathbf{x}$, then
        $\norm{W_{\dec}}_{2, 2} < K_{\dec}$.

\textbf{6. Nonlinear blocks are either MLP or GLU.}
        For every $i \in [L]$, $\mathcal{F}_i$ is either
        a family of MLP layers or a family of GLU layers.
        In the former case, all elements of $\mathcal{F}_i$
        are MLP layers with $M_i$ layers and
        with the same activation functions $\rho_i$ which is either ReLU or a sigmoid function which satisfies the following:
        $\rho_i(0)=0.5$, it is $1$-Lipschitz, $\rho(x) \in [-1,1]$,
        $\rho_i(x)-\rho_i(0)$ is odd.
        If $\mathcal{F}_i$ is a family of GLU layers, then
        all its elements have the sigmoid function $\sigma_i$.

\textbf{7. Bounded weights for MLP.}
        For every $i \in [L]$ such that $\mathcal{F}_i$ is a family of MLP layers, there exists $K_{W,i},K_{\bfb,i}$ such that
        for every
        $f \in \mathcal{F}_{i}$
        of the form \eqref{def:mlp:eq1} with $M=M_i$ and $\rho=\rho_i$, the weights of $f$ satisfy
        \begin{align*}
        & \max\limits_{j \in [M
        + 1]} 
        \norm{W_j}_{\infty, \infty} < K_{W,i}, \quad  \max\limits_{j \in  [M + 1]} \norm{\bfb_j}_{\infty} < K_{\bfb,i}.
        \end{align*}

\textbf{8. Bounded weights for  GLU}
        For every $i \in [L]$ such that $\mathcal{F}_i$
        is a family of GLUs, there exists a constant
        $K_{\text{GLU},i}$ such that for every $f \in \mathcal{F}_i$ of the form \eqref{def:glu:eq1} with $\sigma=\sigma_i$
        $\norm{W}_{\infty, \infty} <
        K_{\text{\normalfont GLU,i}}$.
\end{assumption}

The first assumption is a standard one, the only restriction
is that all deep SSMs have the same depth and all
SSM blocks have the same residual connection.  \\
Assumption 2, though being restrictive, covers key scenarios such as classification and 1-dimensional regression, which are central to theoretical analysis.\\
Assumption 3 requires the loss function to be Lipschitz-continuous, which is a standard assumption in machine learning and holds for most of the loss functions used in practice, including the squared loss on bounded domains, the $\ell_1$ loss and the cross-entropy loss. This ensures boundedness during the learning process.\\
Assumption 4 is also fairly standard, as input normalization is common in practice.\\
Assumption 5 is the key assumption enforcing SSM stability via structured parametrization. Beyond numerical benefits, stability ensures reliable predictions by preventing small input changes from causing large output variations, crucial for learning and inference. \\
Assumptions 6 and 7 are again considered standard, requiring non-linear layers to be either all MLPs or all GLUs with specific activations and enforcing bounded weights for the encoder, decoder, and MLP/GLU layers.



\section{Main results}
\label{sec:main}
We derive a Rademacher complexity-based generalization bound for deep SSM models, independent of sequence length. The key challenges are:
\\
\textbf{(1)} bounding the Rademacher complexity of SSMs,
\textbf{(2)} extending this to hybrid SSM-neural network blocks, and
\textbf{(3)} handling deep architectures with multiple such blocks.
For stable SSMs, we show their norm bounds the Rademacher complexity for any sequence length. To address the second and third challenges, we introduce \emph{Rademacher Contraction}, a universal framework that enables componentwise complexity estimation in deep models.

\begin{definition}[$(\mu, c)$-\emph{Rademacher Contraction}]
  \label{def:rc}
Let $X_1$ and $X_2$ be subsets of Banach spaces $\mathcal{X}_1,\mathcal{X}_2$,
with norms $\|\cdot \|_{\mathcal{X}_1}$ and $\|\cdot \|_{\mathcal{X}_2}$, and
let $\mu \geq 0$ and $c \geq 0$.
A set of functions $\Phi
= \{ \varphi: X_1 \to X_2 \}$ is said to be $(\mu, c)$-\emph{Rademacher Contraction},
or $(\mu, c)$-RC in short, if for all $n \in \mathbb{N}^+$ and $Z
\subseteq X_1^n$ we have
\begin{equation}
\label{eq:rc}
\begin{split}
  & \mathbb{E}_{\sigma}\left[
      \sup\limits_{\varphi \in \Phi}
       \sup\limits_{\{\uu_{i}\}_{i=1}^{n} \in Z}
      \norm{\frac{1}{N}\sum\limits_{i=1}^{N}
      \sigma_i \varphi(\uu_{i})}_{\mathcal{X}_2}
   \right]
   \leq  \\ \mu
  & \mathbb{E}_{\sigma}\left[
      \sup\limits_{\{\uu_{i}\}_{i=1}^{n} \in Z}
      \norm{\frac{1}{N}\sum\limits_{i=1}^{N}
       \sigma_i \uu_{i}}_{\mathcal{X}_1}
   \right] + \frac{c}{\sqrt{N}},
\end{split}
\end{equation}
where $\sigma_i$ are i.i.d. Rademacher random variables, $i \in [N]$, i.e.
$\mathbb{P}(\sigma_i = 1) = \mathbb{P}(\sigma_i = -1) = 1/2$ and $\mathbb{E}_{\sigma}$ denotes the
expected values w.r.t. the random variables
$\{\sigma_i\}_{i=1}^{N}$.
\end{definition}

\textbf{Rademacher Contractions in the literature.}  While the concept of RC is new, special cases of  Definition \ref{def:rc} were used in \cite{golowich2018size,truong2022rademacher} for bounding Rademacher complexity of deep neural networks,
see Appendix \ref{rc_literature} for a detailed comparison.

\begin{table*}[t]
\caption{Table of $(\mu,c)$ constants.}
\label{tableRC} 
\label{sample-table}
\begin{center}
\begin{small}
\begin{sc}
\begin{tabular}{cccc}
\toprule
Linear layer  &  $\mu(r)$  & $c(r)$ &  $\hat{r}(r)$ \\
\midrule
\hline
  $\mathcal{F}_{\enc}$ & $ K_{\enc}$ & 0 & $K_{\enc}r$ \\
  $\mathcal{F}_{\dec}$ & $K_{\dec}$ & 0 & $K_{\dec}r$ \\
  $\mathcal{E}$ defined on $\ell^{\infty}_T(\mathbb{R}^{n_u})$
  & $K_{1} $ & 0  & $K_1r$  \\
  $\mathcal{E}$ defined on $\ell^{2}_T(\mathbb{R}^{n_u})$ &
  $K_{2} $ & 0  & $K_2r$  \\
   $\{ f^{\pool} \}$ & $1$ & $0$ & $r$ \\
   \hline
   SSM block &  & &  \\
$\mathcal{F}_{1}^{\text{\normalfont DTB}} $& $\begin{array}{l}
K_2\mu_1(K_2 r)
+ \alpha_1\end{array}$ & $c_1(K_2r)$ & $\begin{array}{l} \hat{r}_i(K_2r)
+ \alpha_1 r \end{array}$  \\
   $\mathcal{F}_{i > 1}^{\text{\normalfont DTB}}$ & $\begin{array}{l}
   K_1\mu_i(K_1r)
   + \alpha_i \end{array}$ & $c_i(K_1r)$ & $\begin{array}{l} \hat{r}_i(K_1r)  +\alpha_i r\end{array}$  \\
   \midrule
   Non-linear layer & $\mu_i(r)$ & $c_i(r)$ & $\hat{r}_i(r)$ \\
   \midrule
   \hline
  $\mathcal{F}_i$ is  MLP w. ReLU &
 $(4K_{W,i})^{M_i+1} $  & $\begin{array}{l} 4K_{\bfb,i} \cdot
   \sum\limits_{q=1}^{M_i} (4K_{W,i})^q \end{array}$
   &
   $\begin{array}{l} K_{W,i}^{M_i+1}r +
   K_{\bfb,i} \cdot
    \sum\limits_{q=1}^{M_i-1}K_{W,i}^q
   \end{array}$ \\
 $\mathcal{F}_i$ is  MLP w. sigmoid & $K_{W,i}^{M_i+1} $  & $\begin{array}{l} (K_{\bfb,i}
   +
   0.5) \cdot
   \sum\limits_{q=1}^{M_i} (K_{W,i})^q \end{array}$ & $ \begin{array}{l} K_{W,i} +  K_{\bfb,i} \end{array}$ \\
$\mathcal{F}_i$ is  GLU & $\begin{array}{l}
    16(r z^2+z) \\
    z= K_{\text{\normalfont GLU},i} +1
   \end{array}$  &  $0$  & $r$  \\
\bottomrule
\end{tabular}
\end{sc}
\end{small}
\end{center}
\vskip -0.1in
\end{table*}

  \textbf{Interpretation of  RC inequality. \eqref{eq:rc}}
 allows relating the Rademacher complexity of a model class to the Rademacher complexity of its inputs via the constants $\mu$ and $c$. These constants depend on the model class as well as the domain $X_1$ and range $X_2$ of the models. As shown next, the RC property is preserved under the composition of layers.
 \begin{lemma}[Composition lemma]
    \label{lem:composition}
   Let $\Phi_1 = \{\varphi_1: X_1 \to X_2\}$ be $(\mu_1, c_1)$-RC and $\Phi_2 =
   \{\varphi_2: X_2 \to X_3\}$ be $(\mu_2, c_2)$-RC. Then the set of
   compositions $\Phi_2 \circ \Phi_1 := \{ \varphi_2 \circ \varphi_1  \mid \varphi_1 \in \Phi_1, \varphi_2 \in \Phi_2 \}$ is $(\mu_1
   \mu_2,\mu_2 c_1 + c_2)$-RC.
  \end{lemma}
  The proof is in Appendix \ref{app:proofs}.
 Consequently, for deep models composed of layers that each satisfy the RC property, the entire model class is RC as well. Then Equation \eqref{eq:rc} can be applied to bound the Rademacher complexity of the deep model by that of the input sample. The latter can often be bounded, for instance:
\begin{lemma}
\label{inp:bound}
\( \mathbb{E}_{\sigma}\left[
        \norm{\frac{1}{N}\sum\limits_{i=1}^N
        \sigma_i \uu_i}
        _{\ell_T^{2}(\mathbb{R}^{\nin})}\right]
      \le \frac{K_{\uu}}{\sqrt{N}} \)
      for all $\norm{\uu_i} \in B_{\ell_T^2(\mathbb{R}^{\nin})}(K_{\uu})$, $i \in [N]$.
\end{lemma}
The proof follows a standard argument, e.g. see Lemma 26.10 in \cite{shalev2014understanding}, for completeness it is presented in Appendix \ref{app:proofs}.
  That is, in order to bound the Rademacher complexity of deep SSMs, all we need to show is that each component of a deep SSM model is $(\mu,
  c)$-RC for some $\mu$ and $c$ with compatible domains and ranges.
  To this end, for each $i \in [L]$,
  define the \emph{family
  $\mathcal{F}_i^{\text{\normalfont DTB}}$
  of $i$th SSM blocks} as the family of all SSMs blocks $f^{\text{\normalfont DTB}}$ of the form \eqref{def:ssmblock:eq1} such that
    $g \in \mathcal{F}_i$, $\Sigma \in \mathcal{E}$,
    $\alpha=\alpha_i$. In particular, for any $f \in \mathcal{F}$  of the form \eqref{def:dtdeepssm:eq1},
    the $i$th SSM block $f^{\text{B}_i}$ belongs to $\mathcal{F}_i^{\text{DTB}}$.
  \begin{lemma}
    \label{lem:rc}
    For each family $\mathcal{X} \in \mathcal{F}_{\enc}$,
    $\mathcal{F}_{\dec}$, $\mathcal{F}_{i}, \mathcal{F}^{\text{\normalfont DTB}}_i, i \in [L],  \mathcal{E},\{f^{\text{\normalfont Pool}}\}$ of models,
    the family $\mathcal{X}|_{B(r)}$, i.e. the elements of
    $\mathcal{X}$ restricted to a ball of radius $r$,
    is $(\mu(r),c(r))$-RC
    in their domain. The range of
    the elements of $\mathcal{X}|_{B(r)}$
    is a ball of
    radius $\hat{r}(r)$,  as described in Table \ref{tableRC}.
\end{lemma}

The proof is in Appendix \ref{app:proofs}.
The lemma implies that an SSM layer can
only increase the input's complexity by the factor $\norm{\Sigma}_p$, $p=1,2$, and the latter gets smaller as the system gets more stable.
The results on the MLP layers rely on proof techniques from \cite{truong2022rademacher,truong2022generalization} used to bound their Rademacher complexity. These bounds on MLP layers are considered conservative, however improving existing bounds on the Rademacher complexity of MLPs are out of the scope of this paper.
The proof for the GLU layer is similar to that of the MLP layer, although handling the elementwise product in GLU requires some additional steps. In contrast to other layers, for GLU
the values of  $\mu,c$ depend on the magnitude of the inputs.

 Using Lemma \ref{lem:rc} and Lemma \ref{inp:bound} and classical Rademacher complexity based PAC bounds, e.g. see Theorem 26.5 \cite{shalev2014understanding}, leads to a PAC bound for deep SSMs, summarized in the main theorem below.
 \begin{theorem}[Main]
   \label{thm:maingeneral}
    Let Assumption \ref{ass:all} hold. Then
   \begin{equation}
       \label{eq:main}
       \begin{split}
       & \mathbb{P}_{S \sim \mathcal{D}^N}
       \Biggl(
            \forall f \in \mathcal{F}: \quad \mathcal{L}(f)
            - \mathcal{L}_{emp}^S(f)
           \leq  \\
           &  \frac{\mu K_\uu L_l + c L_l}{\sqrt{N}}
            + K_l\sqrt{\frac{2\log(4/ \delta)}{N}}
       \Biggr) > 1-\delta
   \end{split}
   \end{equation}
   where $K_l=2L_l \max\{K_{\dec}r_L, K_y\}$. The term $r_L$ is obtained recursively
   for all $i \in [L]$,
   \begin{equation}
   \label{ri_eq}
   r_i=\left\{ \begin{array}{ll}
       K_{\enc} K_{\uu} & i=1 \\
      \hat{r}_{1}(K_2 r_{1}) + \alpha_{1} r_{1}
      & i=2 \\
       \hat{r}_{i-1}(K_1 r_{i-1}) + \alpha_{i-1} r_{i-1}
      & i > 2
      \end{array}\right.
\end{equation}
where $\hat{r}_i(r)$ are as in Table \ref{tableRC} of Lemma \ref{lem:rc}.
 Moreover, let us define $\mu_1=\mu_1(K_2r_1)$,
  $c_1=c_1(K_2r_1)$,
  and for $i > 1$,
 $\mu_i=\mu_i(K_1r_1)$, $c_i=c_i(K_1r_i)$.
  Finally,
   \begin{equation}
   \label{muc:const}
      \begin{split}
      & \mu = K_{\enc} K_{\dec}
   \left(\mu_{1} K_2  + \alpha_1
     \right)
     \prod\limits_{i=2}^L \left(
     \mu_{i}K_1
    + \alpha_i
     \right) \\
     & c = K_{\dec}
     \sum\limits_{j=1}^L \left[
       \prod\limits_{i=j+1}^L
        \left(\mu_{i}K_1
    + \alpha_i\right)
     \right] c_{j}.
 \end{split}
 \end{equation}

 %
\end{theorem}
\begin{proof}[Sketch of the proof]
From standard PAC bounds involving Rademacher complexity (Theorem 26.5 \cite{shalev2014understanding}) and the Contraction Lemma (Lemma 26.9 from
\cite{shalev2014understanding}), it follows
that with probability at least $1-\delta$,
for any $f \in \mathcal{F}$,
\( \mathcal{L}(f)
            - \mathcal{L}_{emp}^S(f)
           \leq  \mathbb{E}_{\sigma}
           \left[
        \sup\limits_{f \in \mathcal{F}}
        \norm{\frac{1}{N}\sum\limits_{i=1}^{N}
        \sigma_i f(\uu_{i})}_{\ell^2_T(\mathbb{R}^{\nin})}\right]
            + K_l\sqrt{\frac{2\log(4/ \delta)}{N}}
\).
From Lemma \ref{lem:composition} and Lemma \ref{lem:rc}, it follows that the restriction of
the elements of
$\mathcal{F}$ to the ball $B_{\ell^2_T(\mathbb{R}^{\nin})}(K_{\uu})$  of radius $K_{\uu}$ in $\ell^2_T(\mathbb{R}^{\nin})$ is $(\mu, c)$ with $\mu$ and $c$ as in the statement of the Theorem.
Hence,
\( \mathbb{E}_{\sigma}\left[
        \sup\limits_{f \in \mathcal{F}}
        \norm{\frac{1}{N}\sum\limits_{i=1}^{N}
        \sigma_i f(\uu_{i})}_{\ell^2_T(\mathbb{R}^{\nin})}\right] \le \mu
        \mathbb{E}_{\sigma}\left[
        \norm{\frac{1}{N}\sum\limits_{i=1}^N
        \sigma_i \uu_i}_
        {\ell_T^{2}(\mathbb{R}^{\nin})}\right]+\frac{c}{\sqrt{N}}  \le \frac{\mu K_{\uu}+c}{\sqrt{N}}
\).
The complete proof can be found in Appendix \ref{app:proofs}.
\end{proof}
\textbf{Discussion and interpretation of Theorem \ref{thm:maingeneral}.}
The bound vanishes as $N$ grows and remains independent of the sequence length and state dimension — an advantage over typical sequential model bounds that diverge with $T$.

The key constants are $\mu$ and $c$. Note, that for deep SSMs with no MLP layers, $c$ is zero. Intuitively, $\mu$ and $c$ are somewhat analogous to Lipschitz constants of deep networks, although the formal relationship between the two requires future work.

The bound appears exponential in depth, like in case of deep neural networks \cite{truong2022rademacher,golowich2018size}, and includes MLP bounds when MLP layers are used \cite{golowich2020size,truong2022rademacher}. However, SSM layer norms can mitigate this effect: if SSMs are contractions, they counteract the higher Rademacher complexity of nonlinear layers. SSM norms depend on stability — more stable SSMs typically have smaller norms if $C$ and $B$ remain unchanged, suggesting that stability may offset depth effects in both SSM and MLP layers. Moreover, using GLU instead of MLP layers may reduce generalization gaps, as deep SSMs with GLU layers have $c = 0$.

\section{Numerical example}
\label{sec:experiment}
In order to illustrate our results,
we trained a model consisting of a single SSM layer on a binary classification problem of separating the elements of two intertwined spirals, with a training set containing $N$ sequences of length $T$, for various values of $N$.
The training set is depicted in Fig. \ref{fig:spiral},
the details are presented in Appendix \ref{num}.
We computed the bound of Theorem \ref{thm:maingeneral}
for a family of models which contains all models encountered during training.
Figure \ref{fig:pac} illustrates that \textbf{(1)} that the bound holds true, \textbf{(2)} it is non-vacuous, i.e. there exists a value of the true loss - highlighted by the red broken line - which is greater, than the value of the estimation at a different value of $N$.
In addition, the behavior of the generalization gap through learning epochs also conforms to the proposed result, too. The details are presented in Appendix \ref{num}.
\begin{figure}[H]
    \centering
    \includegraphics[width=.48\textwidth]{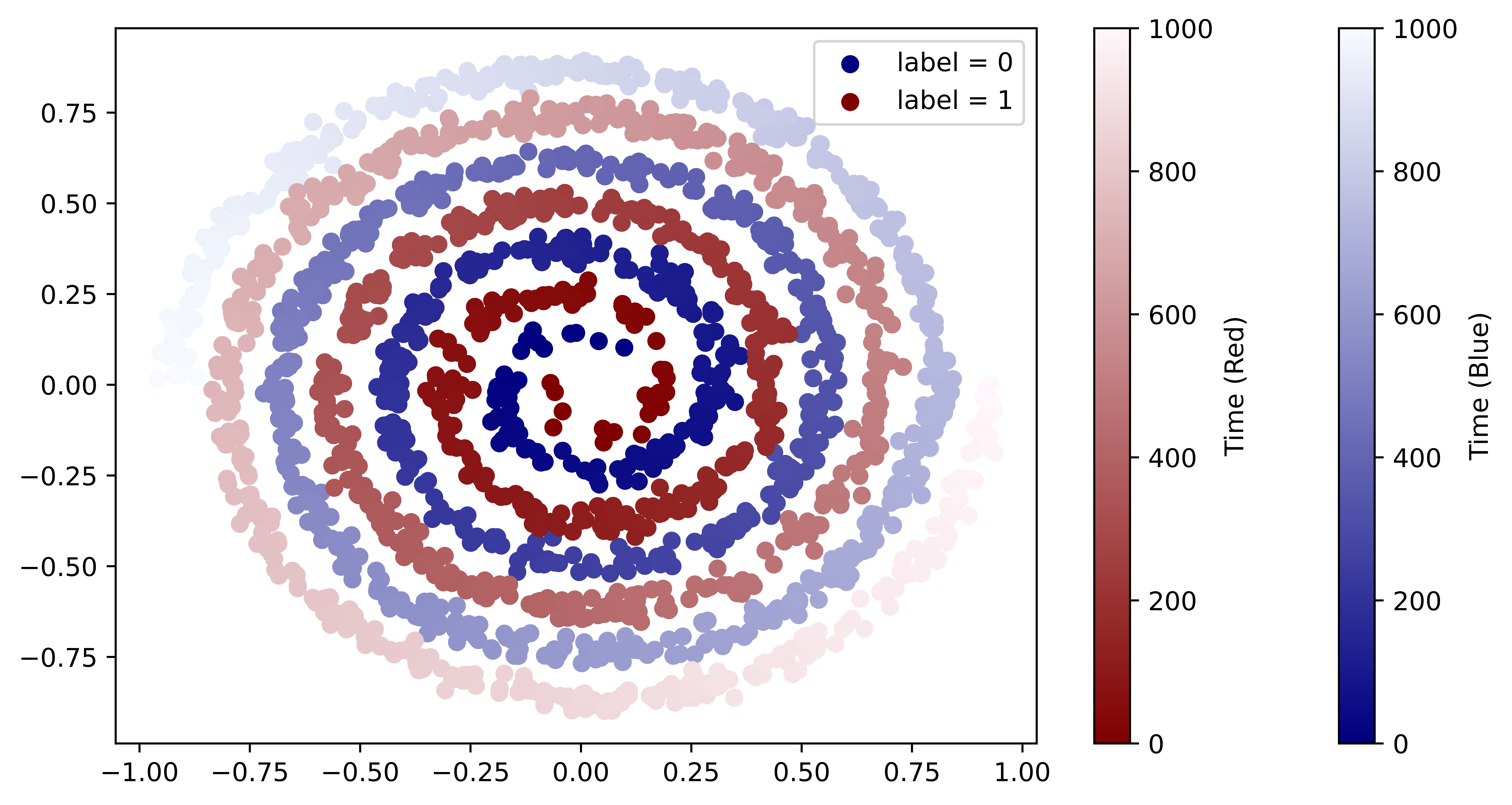}
    \caption{Dataset containing two classes of spiral curves.}
    \label{fig:spiral}
\end{figure}
\begin{figure}[H]
    \centering
    \includegraphics[width=.48\textwidth]{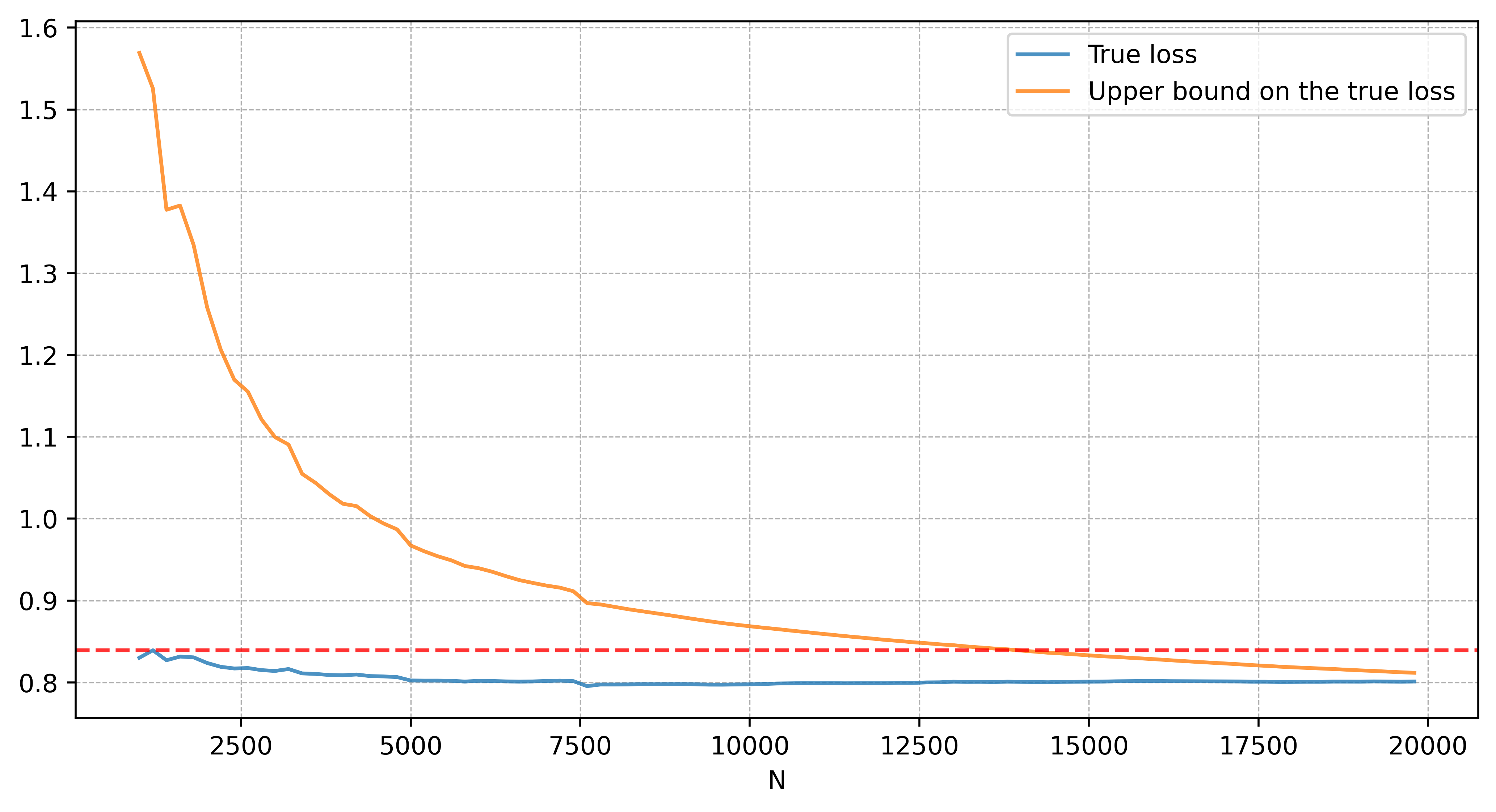}
    \caption{Upper bound on the true loss by taking the empirical loss and the bounding term from Theorem \ref{thm:maingeneral} for various values of $N$.}
    \label{fig:pac}
\end{figure}

\section{Conclusions}
\label{sec:conclusions}

We derive generalization bounds for deep SSMs by decomposing the architecture into components satisfying the definition of \emph{Rademacher Contraction}. Under reasonable stability conditions, the bound is sequence-length independent and improves on prior results for linear RNNs as those bounds depend on the sequence-length exponentially. Given that stability is central to state-of-the-art SSMs (e.g., S4, S5, LRU), our work offers insight into their strong long-range performance.

We introduce the concept of \emph{Rademacher Contraction} which we believe is a powerful tool for determining the complexity of a wide variety of stacked architectures including feedforward and recurrent elements.

Our contraction-based approach, while reasonable, may yield conservative bounds for complex SSMs or deep stacks. However, with state-of-the-art SSMs using fewer than ten blocks, this is a minor concern. A key limitation is that all elements must be \emph{Rademacher Contractions}, which may be too restrictive for complex functions.

Future work includes extending these results to learning from limited (possibly single) time-series and deriving tighter bounds, potentially using concentration inequalities for mixing processes and the PAC-Bayesian framework.
Incorporating Mamba-like architectures into our framework can also be a subject of future research.




\bibliography{lpv,tansens,ssm}
\bibliographystyle{icml2025}

\newpage
\appendix
\onecolumn

\section{Related work on PAC-Bayesian bounds, on finite sample bounds and on PAC bounds for non i.i.d. data}
\label{app:related}

\textbf{PAC-Bayesian bounds for dynamical systems.}
PAC-Bayesian bounds for various classes of
dynamical systems were developed in
\cite{alquier2012pred,alquier2013prediction,shalaeva2019improved,haddouche2022online, haussmann2021learning,haddouche2022pac,
seldin2012pac,abeles2024generalizationboundsmixingprocesses}).
The main difference between the cited papers and the present one are as follows.
\begin{enumerate}
\item{\emph{Single time-series vs. multiple independently sampled time-series.}}
All the cited papers assume that the data used for computing the empirical error is sampled from one single time series. The latter assumption required the use of various extensions of well-known concentration inequalities to the non-i.i.d. case. In particular, the obtained bounds all depend on some mixing coefficients.
In contrast to the cited papers, the present paper assumes multiple i.i.d. samples of time-series', so formally, the learning problem of the present paper is completely different from
the one of the papers cited above.
\item{\emph{PAC-Bayesian vs. PAC bounds.}}
The present paper presents a PAC bound, not a PAC-Bayesian one.
PAC bounds have the advantage that they tend to be simpler to use and interpret, and they provide a uniform bound on the generalization gap, but they also tend to be fairly conservative.
PAC-Bayesian bounds are more involved, they are sensitive to priors and they bound the average (w.r.t. some posterior) generalization gap, only.
However, they are potentially less conservative.
This means that PAC bounds might actually be competitive with PAC-Bayesian ones in situations where the former is easy to evaluate and there are no obvious candidates for suitable priors.
We believe that SSMs might fall in this category: the proposed PAC bound is easy to evaluate, and the choice of a suitable prior is far from obvious.
\item{\emph{Different model classes.}}
The classes of dynamical systems in  \cite{alquier2012pred,alquier2013prediction,shalaeva2019improved,haddouche2022online, haussmann2021learning} do not include state-space processes with partially observed state.
The bound \cite{eringis2024pacbayesian} could, in principle, be applied to a one block SSM without non-linearities, and  \cite{eringis2023pacbayesian} can be applicable to multi-block SSMs in case the latter satisfies some stability conditions which are more stringent than the one in this paper.
However, the application of \cite{eringis2023pacbayesian,eringis2024pacbayesian} is possible
only if the data used for learning is sampled from a single time series.  The same is true for \cite{abeles2024generalizationboundsmixingprocesses}.
\end{enumerate}

\textbf{Finite-sample bounds for dynamical systems.}
  In recent years there has been a significant interest in deriving bounds on the true loss for dynamical systems for particular learning algorithms \cite{oymak2019generalization,oymak2021revisiting,simchowitz2019learning,lale2020logarithmic,foster2020learning,Ziemann_Tu_2022,Ziemann_Sandberg_Matni_2022,Pappas1,ziemann2024sharp}.
  However, most of these papers consider learning from one single time-series. Notable exceptions are
  \cite{JMLR:v25:23-1145,zheng2020non,pmlr-v120-sun20a}, where bounds for the true risk for linear State-Space Models were derived. However, there the derived bound  does not relate the empirical loss to the true one, and it is applicable only for linear dynamical systems, i.e., one block SSM. Moreover, the derived bound is specific to the learning algorithm employed. The latter is based on least-squares solution to linear regression, and it does not seem to be directly applicable to deep SSMs with non-linear blocks.
  In contrast, the results of the present paper are applicable to deep SSMs and to any learning algorithms.

\textbf{PAC bounds for non i.i.d. data.}
  There is a significant body of literature on PAC bounds involving Rademacher complexity
\cite{mcdonald2017rademachercomplexitystationarysequences,Mohri,kuznetsov2017generalization} or other complexity measures for
  non i.i.d data, including data chosen from
  a single time-series. As it was mentioned above, in this paper we consider a different learning problem, namely, learning from multiple independently sampled time series, as
  opposed to one single time series. Moreover, the cited papers propose PAC bounds which involve various measures of the complexity of the parameterization, e.g., Rademacher complexity, but they do not dwell on estimating
  those measures for various classes of dynamical systems, such as SSMs.

\section{Deep SSM architectures}
\label{appendix_ssm}
\begin{remark}[Relationship between discrete-time SSMs  \eqref{eq:dtlti} and continuous-time SSMs]
\label{rem:cont-time}
 In the literature, the SSM layer is often
 defined as a continuous-time system
 \begin{equation}
 \label{ct:ssme:eq1}
  \dot x_c(t) = A_cx_c(t)+B_cv(t), ~ y_c(t)=Cx_c(t)+Dv(t), ~  x_c(0)=0
 \end{equation}
 where $t \in [0,\infty)$.
 In order to transform \eqref{ct:ssme:eq1} to a
 model mapping sequences to sequences, it is discretized in time \cite{mamba,S4,S4D,S5}.
 That is, the following
 discrete-time system is considered:
\begin{equation}
 \label{ct:ssme:eq2}
  x[k+1] = A(\Delta_k)x[k]+B(\Delta_k)\uu[k], ~ y[k]=Cx[k]+Du(k), ~  x[1]=0
 \end{equation}
 such that the matrix valued functions $A(\Delta)$ and $B(\Delta)$
 are defined as
 $A(\Delta)=e^{A_c \Delta}$,
 $B(\Delta)=\int_0^{\Delta} e^{A(\Delta-s)}B ds$,
 $\Delta_k=\Delta(\uu[k])$ is a function of $\uu[k]$, and if $v(t)=\uu[k]$ for all
 $t \in (\sum_{i=1}^{k-1} \Delta_{i},\sum_{i=1}^{k} \Delta_i]$, then $x[k]=x_c(\Delta_{k-1})$,
 $y[k]=y_c(\Delta_{k-1})$, $k \in [T]$,
 and $\Delta_1:=0$.
 If $\Delta$ depends on the input as in \cite{mamba}, one obtains a discrete-time
\emph{linear parameter-varying system (LPV)} \cite{toth2010modeling}, or a so-called selective State-Space Model.
If $\Delta_k$ equals a constant $\Delta$, then
\eqref{ct:ssme:eq2} is an LTI system of the form
\eqref{eq:dtlti} with $A=A(\Delta)$ and $B=B(\Delta)$.
\end{remark}
\begin{remark}[Stability assumptions in the literature]
\label{stab}
In the literature, it is often assumed that the continuous-time SSM is stable, see Table \ref{tab_summary}.
\begin{table}[H]
\begin{center}
 \begin{tabular}{ |c | c c | }
 \hline
 Model &  SSM  & Block \\
 \hline
 \makecell{S4 \cite{S4}} & \makecell{LTI, $A = \Lambda - P Q^{*}$ \\ block-diagonal, \textbf{stable}}
 & \makecell{SSM +\\ nonlinear activation}  \\
 S4D \cite{S4D} & \makecell{LTI, $A = -\text{exp}(A_{Re}) + i \cdot A_{Im}$ \\
 block-diagonal, \textbf{stable}} & \makecell{SSM +\\ nonlinear activation}  \\
 S5 \cite{S5} & LTI, \textbf{stable} diagonal $A$  & \makecell{SSM +\\ nonlinear activation}  \\
 \makecell{LRU \cite{LRU}} & \makecell{LTI,  diagonal $A$\\ \textbf{stable} complex exponential
 parametrization} & \makecell{SSM +\\
 MLP / GLU +\\skip connection} \\
 \hline
\end{tabular}
    \caption{Summary of popular deep SSM models. \label{tab_summary}}
\end{center}
\end{table}
If  \eqref{ct:ssme:eq1} is a stable continuous-time linear system, i.e. $A_c$ is a Hurwitz matrix (all the eigenvalues of $A_c$ have a negative real part), then $A(\Delta)$ is a Schur matrix \cite{AntoulasBook}, i.e., the corresponding discrete-time SSM is stable.
\end{remark}
\begin{remark}[Computing $\norm{\Sigma}_i$,$i=1,2$]
\label{comp:norm}
An upper bound on the
 norm $\norm{\Sigma}_i$, $i=1,2$ can be computed by
 solving a suitable
 a linear matrix inequality (LMI), which is a standard tool in control theory \cite{LMIBook}.
 Moreover, $\norm{\Sigma}_2$ can also be computed
 using Sylvester equations, for
 which standard numerical algorithms exist \cite{AntoulasBook}.
 Alternatively, both norms
  can be computed by taking a
 sufficiently large finite sum instead of the infinite sum used in their
 definition.
 Finally,
 if $\|A\|_2 < \beta < 1$, then an easy calculation reveals that
 $\norm{\Sigma}_1 \le \left(\|D\|_2+\frac{\|B\|_2 \|C\|_2}{1-\beta}\right)$ and
 $\norm{\Sigma}_2 \le \sqrt{\|D\|_F^2+\frac{n_{y} \|B\|_2^2
 \|C\|_2^2}{1-\beta^2}}$.
\end{remark}

\section{Rademacher complexity}

\begin{definition}[Def. 26.1 in \cite{shalev2014understanding}]
    \label{defiradem}
    The Rademacher complexity of a bounded set $\mathcal{A} \subset
     \mathbb{R}^{m}$ 
    is defined as
    \begin{align*}
        R(\mathcal{A}) =
        \mathbb{E}_{\mathbf{\sigma}}\Bigg[\sup_{a \in \mathcal{A}}
        \frac{1}{m} \sum\limits_{i = 1}^{m}\sigma_i a_i \Bigg],
    \end{align*}
    where the random variables $\mathbf{\sigma}_i$ are i.i.d such that
    $\mathbb{P}[\sigma_i = 1]  = \mathbb{P}[\sigma = -1] = 0.5$. The Rademacher
    complexity of a set of functions $\mathcal{F}$ over a set of samples $S =
    \{s_1\dots s_m\}$ is defined as $R_{S}(\mathcal{F}) =
    R(\left\{(f(s_1),\dots,f(s_m)) \mid f \in \mathcal{F} \right\}).$
\end{definition}

The following is a standard theorem we use in the proof.
\begin{theorem}[Theorem 26.5 in \cite{shalev2014understanding}]
    \label{thm:pac}
     Let $L_0$ denote the
    set of functions of the form $(\uu,y) \mapsto l(f(\uu),y)$ for $f \in
    \mathcal{F}$. Let $K_l$ be such that the functions from $L_0$ all take
    values from the interval $[0,K_l]$. Then  for any $\delta \in (0, 1)$  we
    have
    \begin{align*}
        \mathbb{P}_{\mathbf{S} \sim \mathcal{D}^N}\Bigg(
            \forall f \in \mathcal{F}:
        \mathcal{L}(f) - \mathcal{L}^{S}_{emp}(f)
         \leq 2R_{S}(L_{0})  +
    K_l\sqrt{\frac{2 \log (4/\delta)}{N}} \Bigg) \geq 1 - \delta.
    \end{align*}
\end{theorem}
 \subsection{Rademacher Contractions in the literature}
  \label{rc_literature}
  In
  \cite{golowich2018size} the authors considered biasless ReLU networks and
  proved a similar inequality using Talagrand's Contraction Lemma
  \cite{ledoux1991probability}. In \cite{truong2022rademacher}, the author
  considered neural networks with dense and convolutional layers and derived a
  PAC bound via bounding the Rademacher complexity. One of the key technical
  achievements in \cite{truong2022rademacher} is Theorem 9, which is a more general version of the inequality
  in \cite{golowich2018size}. This was then applied to obtain generalization
  bounds for the task of learning Markov-chains in
  \cite{truong2022generalization}, however the generalization error was measured
  via the marginal cost and the $(\mu,c)$-RC type inequality was only applied
  for time-invariant neural networks. In contrast, we prove that along with time
  invariant models, stable SSMs, defined between certain Banach spaces, also
  satisfy equation \eqref{eq:rc} and apply it to deep structures.

  In a recent work \cite{trauger2024sequence}, the authors consider Transformers
  and implicitly establish similar inequalities to \eqref{eq:rc} by bounding
  different kinds of operator norms of the model and managed to extend it to a
  stack of Transformer layers. Besides these similarities, some key differences
  in our work are that Definition \ref{def:rc} provides an explicit way to
  combine SSMs with neural networks, even in residual blocks; we do not assume
  the SSM matrices to be bounded, instead we require the system norm to be
  bounded via stability, which is a weaker condition; and we upper bound the
  Rademacher complexity directly instead of bounding the covering number.

\section{Proofs}
\label{app:proofs}

In this section we need to prove $(\mu, c)$-RC property for linear (or affine)
transformations which are constant in time, in many cases. For better
readability, we only do the calculations once and use it as a lemma.

\begin{lemma}
    \label{lem:affine}
    Let $\mathcal{X}_1,\mathcal{X}_2$ be two Banach spaces with norms $\|\cdot\|_{\mathcal{X}_1}$ and $\|\cdot\|_{\mathcal{X}_2}$, and for every
    bounded linear operation $W:\mathcal{X}_1 \rightarrow \mathcal{X}_2$
    and any $\bfb \in \mathcal{X}_2$
    $f_{W, \underline{\bfb}}(\uu)
    = W(\uu)+ \underline{\bfb} \in \mathcal{X}_2$.
    Let us denote by $\|W\|_{op}$ the induced norm of a bounded linear operator $W:\mathcal{X}_1 \rightarrow \mathcal{X}_2$, i.e.,
    $\norm{W}_{\text{\normalfont op}}:=\sup_{x \in \mathcal{X}_1} \frac{\|W(x)\|_{\mathcal{X}_2}}{\|x\|_{\mathcal{X}_1}}$.
    Let us assume that $W \in \mathcal{W}$ such that $\sup\limits_{W \in \mathcal{W}} \norm{W}_{\text{\normalfont op}} < K_W$
    and $\underline{\bfb} \in \mathcal{B}$ such that $\sup\limits_{\underline{\bfb} \in \mathcal{B}} \norm{\underline{\bfb}}_{\mathcal{X}_2} < K_\bfb$.
     Then the
    set of transformations $\mathcal{F} = \{f_{W, \underline{\bfb}} \mid W \in \mathcal{W},
    \bfb \in \mathcal{B}\}$ is $(K_W, K_\bfb)$-RC, and
    the image of the ball $B_{\mathcal{X}_1}(r)$ under $f \in \mathcal{F}$ is
    contained in $B_{\mathcal{X}_2}(K_W r + K_\bfb)$.
\end{lemma}

\begin{remark}
    \label{rem:affine}
    We are mainly interested in the cases when $\mathcal{X}_1=\ell_T^{q}(\mathbb{R}^{n_u})$, and $\mathcal{X}_1=\ell_T^{l}(\mathbb{R}^{n_y})$,
    $(q,l) \in \{(2,2), (2, \infty), (\infty, \infty)\}$.
    For the special case of affine transformations that are constant in time, i.e. $f(\uu)[k] = W\uu[k] + \bfb$ for a weight matrix $W \in \mathbb{R}^{n_v \times n_u}$ and bias term $\bfb \in \mathbb{R}^{n_v}$ for all $k \in [T]$, the operator norm equals the corresponding matrix norm, i.e $\norm{W}_{\text{\normalfont op}} = \norm{W}_{q,\infty}$. In this case, $\underline{\bfb}$ is the sequence for which $\underline{\bfb}[k] = \bfb$ for all $k \in [T]$, thus $\norm{\underline{\bfb}}_{\ell^{\infty}_T(\mathbb{R}^{n_v})} = \norm{\bfb}_{\infty}$.
\end{remark}

\begin{proof}
    First, let us prove a simple fact about Rademacher random variables that we
    will need, namely if $\sigma = \{\sigma_i\}_{i=1}^{N}$ is a sequence of
    i.i.d. Rademacher variables, then
    \begin{equation}
        \label{eq:rademacher}
        \mathbb{E}_{\sigma}\left[
            \left|\sum\limits_{i=1}^{N} \sigma_i \right|
         \right] \leq \sqrt{N}.
    \end{equation}
    This is true, because
    \begin{align*}
        &\mathbb{E}_{\sigma}\left[
            \left|\sum\limits_{i=1}^{N} \sigma_i \right|
         \right]
         =
        \sqrt{\left(\mathbb{E}_{\sigma}\left[
            \left|\sum\limits_{i=1}^{N} \sigma_i \right|
         \right]\right)^2} \leq
        \sqrt{\mathbb{E}_{\sigma}\left[
            \left|\sum\limits_{i=1}^{N} \sigma_i \right|^2
         \right]} \\
         &=
        \sqrt{\mathbb{E}_{\sigma}\left[
            \sum\limits_{i=1}^{N} \sigma^2_i
            + 2 \sum\limits_{i,j = 1}^N \sigma_i \sigma_j
         \right]}
         =
        \sqrt{
            \sum\limits_{i=1}^{N} \mathbb{E}_{\sigma}\left[\sigma^2_i  \right]
            + 2 \sum\limits_{i,j = 1}^N \mathbb{E}_{\sigma}\left[\sigma_i
             \sigma_j
         \right]} = \sqrt{N},
    \end{align*}
    where the first inequality follows from Jensen's inequality and the last
    equality follows from the linearity of the expectation, and the facts that
    $\sigma_i$ are Rademacher variables and form and i.i.d sample.

    For $Z \in \mathcal{X}_1$
    we have
    \begin{align*}
    &\mathbb{E}_{\sigma}\left[
    \sup\limits_{(W,\underline{\bfb}) \in \mathcal{W} \times \mathcal{B}}
    \sup\limits_{\{\uu_i \}_{i=1}^N \in Z}
    \norm{
    \frac{1}{N}
    \sum\limits_{i=1}^{N}\sigma_i(W(\uu_i) + \underline{\bfb})}
    _{\mathcal{X}_2}
    \right] \\
    &\leq \mathbb{E}_{\sigma}\left[
    \sup\limits_{W \in \mathcal{W}}
    \sup\limits_{\{\uu_i \}_{i=1}^N \in Z}
    \norm{
    \frac{1}{N}
    \sum\limits_{i=1}^{N}\sigma_i W(\uu_i)}
    _{\mathcal{X}_2}
    \right] + \mathbb{E}_{\sigma}\left[
    \sup\limits_{\bfb \in \mathcal{B}}
    \norm{
    \frac{1}{N}
    \sum\limits_{i=1}^{N}\sigma_i \underline{\bfb}}
    _{\mathcal{X}_2}
    \right]\\
    & = \mathbb{E}_{\sigma}\left[
    \sup\limits_{W \in \mathcal{W}}
    \sup\limits_{\{\uu_i \}_{i=1}^N \in Z}
    \norm{W\left(
    \frac{1}{N}
    \sum\limits_{i=1}^{N}\sigma_i\uu_i\right)}
    _{\mathcal{X}_2}
    \right] + \mathbb{E}_{\sigma}\left[
    \sup\limits_{\bfb \in \mathcal{B}}
    \norm{
    \frac{1}{N}
    \sum\limits_{i=1}^{N}\sigma_i \underline{\bfb}}
    _{\mathcal{X}_1}
    \right]\\
    &\leq \mathbb{E}_{\sigma}\left[
    \sup\limits_{W \in \mathcal{W}}
    \norm{W}_{\text{\normalfont op}}
    \sup\limits_{\{\uu_i \}_{i=1}^N \in Z}
    \norm{
    \frac{1}{N}
    \sum\limits_{i=1}^{N}\sigma_i  \uu_i }_{\mathcal{X}_1}
    \right] +\mathbb{E}_{\sigma}\left[
    \frac{1}{N}
    \left|\sum\limits_{i=1}^{N}\sigma_i \right|
    \sup\limits_{\bfb \in \mathcal{B}}
    \norm{
    \underline{\bfb}}
    _{\mathcal{X}_2}
    \right]\\
    &\leq
     \sup\limits_{W \in \mathcal{W}}
    \norm{W}_{\text{\normalfont op}}
    \mathbb{E}_{\sigma}\left[
    \sup\limits_{\{\uu_i \}_{i=1}^N \in Z}
    \norm{
    \frac{1}{N}
    \sum\limits_{i=1}^{N}\sigma_i  \uu_i }_{\mathcal{X}_1}
    \right] +\sup\limits_{\bfb \in \mathcal{B}}
        \norm{
    \underline{\bfb}}
    _{\mathcal{X}_2}
    \mathbb{E}_{\sigma}\left[
    \frac{1}{N}
    \left|\sum\limits_{i=1}^{N}\sigma_i \right|
    \right]\\
    &\leq
     \sup\limits_{W \in \mathcal{W}}
    \norm{W}_{\text{\normalfont op}}
    \mathbb{E}_{\sigma}\left[
    \sup\limits_{\{\uu_i \}_{i=1}^N \in Z}
    \norm{
    \frac{1}{N}
    \sum\limits_{i=1}^{N}\sigma_i  \uu_i }_{\mathcal{X}_1}
    \right] + \frac{1}{\sqrt{N}}\sup\limits_{\bfb \in \mathcal{B}}
        \norm{
    \underline{\bfb}}
    _{\mathcal{X}_2}
    \end{align*}
 where the first inequality follows from the triangle inequality, the first
 equality is the linearity of $W$, the second inequality follows from the
 definition of the operator norm, while the third and fourth inequalities refer only to the bias term
 and follow from the absolute homogeneity of the norm and equation \eqref{eq:rademacher}.


 We can see that the calculations hold if the transformations are restricted to
 the ball $B_{X_1}(r)$ for any choice of $X_1$ we consider. The radius can grow
 as
 \begin{align*}
    \norm{W(\uu) + \bfb}_{\mathcal{X}_2} \leq
    \norm{W(\uu)}_{\mathcal{X}_2}
     + \norm{\underline{\bfb}}_{\mathcal{X}_2} \leq
    \norm{W}_{\text{\normalfont op} }\norm{\uu}_{\mathcal{X}_1}
    + \norm{\underline{\bfb}}_{\mathcal{X}_2}.
 \end{align*}

 Remark \ref{rem:affine} is straightforward from the definitions of the considered norms.
\end{proof}

\begin{proof}[Proof of Lemma \ref{lem:composition}]
 Let the Banach spaces which contain $X_i$ be denoted by $\mathcal{X}_i$ for $i=1,2,3$. Let $Z \subseteq X_1^{N}$ and $\tilde{Z} = \{\{\varphi_1(\uu_i)\}_{i=1}^{N} \mid
  \varphi_1 \in \Phi_1 \}$. We have
  \begin{align*}
      &\mathbb{E}_{\sigma}\left[
          \sup\limits_{\varphi_2 \in \Phi_2}
          \sup\limits_{\varphi_1 \in \Phi_1}
           \sup\limits_{\{\uu_i\}_{i=1}^{N} \in Z}
           \norm{\frac{1}{N}
           \sum\limits_{i=1}^{N} \sigma_i \varphi_2(\varphi_1(\uu_i))}_{\mathcal{X}_3}
      \right]  \\
      &= \mathbb{E}_{\sigma}\left[
          \sup\limits_{\varphi_2 \in \Phi_2}
           \sup\limits_{\{\vv_i\}_{i=1}^{N} \in \tilde{Z}}
           \norm{\frac{1}{N}
           \sum\limits_{i=1}^{N} \sigma_i \varphi_2(\vv_i)}_{\mathcal{X}_3}
      \right] \\
      &\leq \mu_2
       \mathbb{E}_{\sigma}\left[
          \sup\limits_{\varphi_1 \in \Phi_1}
           \sup\limits_{\{\uu_i\}_{i=1}^{N} \in Z}
           \norm{\frac{1}{N}
           \sum\limits_{i=1}^{N} \sigma_i \varphi_1(\uu_i)}_{\mathcal{X}_2}
       \right]
       + \frac{c_2}{\sqrt{N}} \\
       &\leq \mu_2 \mu_1 \mathbb{E}_{\sigma}\left[
          \sup\limits_{\{\uu_i\}_{i=1}^{N} \in Z}
           \norm{\frac{1}{N}
           \sum\limits_{i=1}^{N} \sigma_i \uu_i}_{\mathcal{X}_1}
       \right] + \mu_2 \frac{c_1}{\sqrt{N}} + \frac{c_2}{\sqrt{N}}
  \end{align*}
\end{proof}

\begin{proof}[Proof of Lemma \ref{lem:rc}]
   \textbf{Encoder and decoder.} The encoder is case \textbf{a)}, while the
   decoder is case \textbf{b)} in Lemma \ref{lem:affine} along with Remark \ref{rem:affine}.

    \textbf{SSM.} As discussed in Section \ref{subsec:deepssm}, an SSM is
    equivalent to a linear transformation called its input-output map.
    Therefore, by Lemma \ref{lem:affine}, the SSM is $(\mu,  0)$-RC in both
    cases, where $\mu$ is the operator norm of the input-output map. Combining
    this with Lemma \ref{lem:normest} yields the result.

    \begin{remark}
        As the value of $T$ is fixed, the input-output map can be described by
        the so-called Toeplitz matrix of the system. In  this case, the operator
        norm equals to the appropriate induced matrix norm of the Toeplitz
        matrix. For the case of $T=\infty$, the input-output map still exists
        and is a linear operator. The proof of Lemma \ref{lem:affine} holds in
        this case as well for operator norms.
    \end{remark}

    \textbf{Pooling.}
    For any $Z \subseteq \ell_T^{\infty}(\mathbb{R}^{n_u})$ we have
    \begin{align*}
    &\mathbb{E}_{\sigma}\left[
    \sup\limits_{\{\z_i \}_{i=1}^N \in Z}
    \norm{
    \frac{1}{N}
    \sum\limits_{i=1}^{N}\sigma_i f^{\pool}(\z_i)}
    _{\infty}
    \right] \\
    &=
    \mathbb{E}_{\sigma}\left[
    \sup\limits_{\{\z_i \}_{i=1}^N \in Z}
    \sup\limits_{1 \leq j \leq n_u}
    \left | \frac{1}{N}
    \sum\limits_{i=1}^{N}\sigma_i
    \left(\frac{1}{T} \sum\limits_{k=1}^{T} \z_i^{(j)}[k]\right) \right |
    \right] \\
    &=
    \mathbb{E}_{\sigma}\left[
    \sup\limits_{\{\z_i \}_{i=1}^N \in Z}
    \sup\limits_{1 \leq j \leq n_u}
    \left |\frac{1}{T}
    \sum\limits_{k=1}^{T}
    \left(\frac{1}{N} \sum\limits_{i=1}^{N} \sigma_i \z_i^{(j)}[k]  \right)
     \right |
    \right] \\
    &\leq
    \mathbb{E}_{\sigma}\left[
    \sup\limits_{\{\z_i \}_{i=1}^N \in Z}
    \frac{1}{T}
    \sum\limits_{k=1}^{T}
    \sup\limits_{1 \leq j \leq n_u}
    \left|\frac{1}{N} \sum\limits_{i=1}^{N} \sigma_i \z_i^{(j)}[k]\right|
    \right] \\
    &=
    \mathbb{E}_{\sigma}\left[
    \sup\limits_{\{\z_i \}_{i=1}^N \in Z}
    \frac{1}{T}
    \sum\limits_{k=1}^{T}
    \norm{\frac{1}{N} \sum\limits_{i=1}^{N} \sigma_i \z_i[k]}_\infty
    \right] \\
    &\leq
    \mathbb{E}_{\sigma}\left[
    \sup\limits_{\{\z_i \}_{i=1}^N \in Z}
    \norm{\frac{1}{N} \sum\limits_{i=1}^{N} \sigma_i \z_i}
    _{\ell_T^{\infty}(\mathbb{R}^{n_u})}
    \right] \\
    \end{align*}

\textbf{MLP.}
 For both type of activation functions we will prove the result by first proving it for single layer networks.
 To this end, let $\rho$ be
 an activation function, which is either ReLU
 or a sigmoid with the properties stated in Assumption \ref{ass:all}.


Consider constants $K_W,K_{\bfb} > 0$ and integers $m,n_v >0$.
We first consider that the family
$\mathcal{F}_{MLP,K_W,K_{\bfb},\rho,m,n_v}$ of single hidden layer
neural networks $f: \ell^{\infty}(\mathbb{R}^{m}) \rightarrow  \ell_T^{\infty}(\mathbb{R}^{n_v})$ defined by
$f(\uu)[k] = \rho(g(\uu[k]))$, where $g(x) = W \x + \bfb$
is the preactivation  function and $g$ belongs to
the set $\mathcal{G}_{K_{W},K_{\bfb},m,n_v} = \left\{ g: \x \mapsto  W\x + \bfb \mid W \in
\mathcal{W}, \bfb \in \mathcal{B} \right\}$,
where $\mathcal{W}=\{ W \in \mathbb{R}^{n_v \times m} \mid \|W\|_{\infty,\infty}< K_W\}$
and
$\mathcal{B}=\{ b \in \mathbb{R}^{n_v} \mid \|b\|_{\infty} < K_{\bfb}\}$.

We will show that  $\mathcal{F}_{MLP,K_W,K_{\bfb},\rho,m,n_v}$ is
$(K_W,K_{\bfb}+0.5)$-RC  if $\rho$ is sigmoid, and
it is $(4K_W,4K_{\bfb})$-RC if $\rho$ is ReLU.
Moreover, the elements of $\mathcal{F}_{MLP,K_W,K_{\bfb},\rho,m,n_v}$  map balls or radius $r$ to balls of radius
$\hat{r}(r)$, where
$\hat{r}(r)=K_W+K_{\bfb}$ if $\rho$ is sigmoid, and
$\hat{r}(r)=K_Wr + K_{\bfb}$ if $\rho$ is ReLU.

From this the statement of the lemma can be derived as follows.
 Let $\mathcal{G}_{K_{W,i},K_{\bfb,i},n_{M},n_u}$  be the set of all
 models $f(\uu)[k]=g(\uu[k])$
 such that $g \in G_{K_{W,i},K_{\bfb,i},n_{M},n_u}$.
 Notice that
 $\mathcal{F}_i$ is contained in the composition (as defined in Lemma \ref{lem:composition})
$\mathcal{G}_{K_{W,i},K_{\bfb,i},n_{M},n_u} \circ  \mathcal{F}_{K_{W,i},K_{\bfb,i},n_M,n_{M+1},\rho_i}
 \circ  \cdots \circ \mathcal{F}_{K_{W,i},K_{\bfb,i},n_{y},n_{2},\rho_i}
 $
 for suitable integers $n_j$, $j \in [M+1]$.
From the discussion above,
$\mathcal{F}_{K_{W,i},K_{\bfb,i},n_j,n_{j+1},\rho_i}$
is $(K_{W,i},K_{\bfb,i}+0.5)$-RC
(sigmoid)
or  $(4K_W,4K_{\bfb})$-RC (ReLU)
and its elements map balls of radius $r$
to balls of radius $1$ (sigmoid)
or $K_Wr+K_{\bfb}$ (ReLU).
From Lemma \ref{lem:affine} and Remark \ref{rem:affine}
it follows that that
$\mathcal{G}_{K_{W,i},K_{\bfb,i},n_{u},n_M}$ is $(K_{W,i},K_{\bfb,i})$-RC
and its element map ball of radius $r$ to balls of radius $K_{W,i}r+K_{\bfb,i}$.
The statement of the lemma follows now by repeated application of  Lemma \ref{lem:composition}.

It is left to prove the claims for single layer MLPs with sigmoid and ReLU activation functions respectively.

\textbf{Single layer MLP with sigmoid activations.}
Let $\rho$ be a sigmoid such that it is $1$-Lipschitz, $\rho(x) \in [-1,1]$, $\rho(0)=0.5$,
$\rho(x)-\rho(0)$ is odd.

To streamline the presentation, for an input sequence $\z \in \ell_T^{\infty}(\mathbb{R}^{m})$ let $g(\z) \in \ell_T^{\infty}(\mathbb{R}^{n_v})$ mean that we apply $g$ for each timestep
independently, i.e. $g(\z)[k] = g(\z[k])$.
We have

\begin{align*}
    &\mathbb{E}_{\sigma}\left[
    \sup\limits_{g \in \mathcal{G}}
    \sup\limits_{\{\z_i \}_{i=1}^N \in Z}
    \norm{
    \frac{1}{N}
    \sum\limits_{i=1}^{N}\sigma_i \rho(g(\z_i))}
    _{\ell_T^{\infty}(\mathbb{R}^{n_u})}
    \right] \\
    & = \mathbb{E}_{\sigma}\left[
    \sup\limits_{(W, \bfb) \in \mathcal{W} \times \mathcal{B}}
    \sup\limits_{\{\z_i \}_{i=1}^N \in Z}
    \sup\limits_{1 \leq k \leq T}
    \norm{
    \frac{1}{N}
    \sum\limits_{i=1}^{N}\sigma_i \rho(W\z_i[k] + \bfb)}
    _{\infty}
    \right] \\
\end{align*}

Let $\x_i=i$, $i=1,\ldots,N$ and let $\mathcal{H} = \{ h_{W, \bfb,
\underline{z}, k} \mid (W, \bfb, \underline{z}, k) \in \mathcal{W} \times
\mathcal{B} \times (Z \cup \{0\}) \times [T]\}$ such that
$h_{W,\bfb,\underline{z},k}(\x_i) = g(\z_i[k])$. Under our assumptions
$\mathcal{H}$ is symmetric to the origin, meaning that $h \in \mathcal{H}$
implies $-h \in \mathcal{H}$. Indeed, notice that if $(W,\bfb) \in \mathcal{W} \times
\mathcal{B}$ then $(-W, -\bfb) \in \mathcal{W} \times \mathcal{B}$,
and hence
$h_{-W,-\bfb,\underline{z},k}=-h_{W,\bfb,\underline{z},k}$ also belongs to $\mathcal{H}$.
We can
apply Theorem 2 from \cite{truong2022rademacher} for the sigmoid activation
$\rho$ and by using that $\rho(x) - \rho(0)$ is odd, we derive the following.

\begin{align*}
    & \mathbb{E}_{\sigma}\left[
    \sup\limits_{(W, \bfb) \in \mathcal{W} \times \mathcal{B}}
    \sup\limits_{\{\z_i \}_{i=1}^N \in Z}
    \sup\limits_{1 \leq k \leq T}
    \norm{
    \frac{1}{N}
    \sum\limits_{i=1}^{N}\sigma_i \rho(W\z_i[k] + \bfb)}
    _{\infty}
    \right] \\
    &=  \mathbb{E}_{\sigma}\left[
    \sup\limits_{h \in \mathcal{H}}
    \norm{
    \frac{1}{N}
    \sum\limits_{i=1}^{N}\sigma_i \rho(h(\x_i))}
    _{\infty}
    \right] \\
    & \leq
    \mathbb{E}_{\sigma}\left[
    \sup\limits_{h \in \mathcal{H}}
    \norm{
    \frac{1}{N}
    \sum\limits_{i=1}^{N}\sigma_i h(\x_i)}
    _{\infty}
    \right]
    + \frac{1}{2 \sqrt{N}}\\
    &=
    \mathbb{E}_{\sigma}\left[
    \sup\limits_{(W, \bfb) \in \mathcal{W} \times \mathcal{B}}
    \sup\limits_{\{\z_i \}_{i=1}^N \in Z}
    \sup\limits_{1 \leq k \leq T}
    \norm{
    \frac{1}{N}
    \sum\limits_{i=1}^{N}\sigma_i (W\z_i[k] + \bfb)}
    _{\infty}
    \right]
    + \frac{1}{2 \sqrt{N}}\\
    &=
    \mathbb{E}_{\sigma}\left[
    \sup\limits_{(W, b) \in \mathcal{W} \times \mathcal{B}}
    \sup\limits_{\{\z_i \}_{i=1}^N \in Z}
    \norm{
    \frac{1}{N}
    \sum\limits_{i=1}^{N}\sigma_i (W\z_i + \bfb)}
    _{\ell_T^{\infty}(\mathbb{R}^{n_v})}
    \right]
    + \frac{1}{2 \sqrt{N}},
\end{align*}
because the sigmoid is 1-Lipschitz and $\rho(0)=
0.5$. Now we can apply Lemma \ref{lem:affine} (see Remark \ref{rem:affine}) to get that
\begin{align*}
    &\mathbb{E}_{\sigma}\left[
    \sup\limits_{(W, \bfb) \in \mathcal{W} \times \mathcal{B}}
    \sup\limits_{\{\z_i \}_{i=1}^N \in Z}
    \norm{
    \frac{1}{N}
    \sum\limits_{i=1}^{N}\sigma_i (W\z_i + \bfb)}
    _{\ell_T^{\infty}(\mathbb{R}^{n_v})}
    \right]
    + \frac{1}{2 \sqrt{N}} \\
    & \leq
    \sup\limits_{W \in \mathcal{W}} \norm{W}_{\infty, \infty}
    \mathbb{E}_{\sigma}\left[
    \sup\limits_{\{\z_i \}_{i=1}^N \in Z}
    \norm{
    \frac{1}{N}
    \sum\limits_{i=1}^{N}\sigma_i \z_i}
    _{\ell_T^{\infty}(\mathbb{R}^{n_u})}
    \right]  +
    \frac{1}{\sqrt{N}}\sup\limits_{\bfb \in \mathcal{B}} \norm{\bfb}_{\infty}
    + \frac{1}{2 \sqrt{N}}
\end{align*}

Therefore, the sigmoid MLP layer is $(K_W, K_\bfb + 0.5)$-RC. The restriction
 of an MLP to the ball $B_{\ell_T^{\infty}(\mathbb{R}^{n_u})}(r)$  maps to
 the ball $B_{\ell_T^{\infty}(\mathbb{R}^{n_v})}(1)$, because of the
 elementwise sigmoid activation.


\textbf{Single layer MLP with ReLU activations.}
The proof is the
same as in the sigmoid case up to the first inequality. Here we can apply
Equation 4.20 from \cite{ledoux1991probability} (this is the same idea as in the
proof of Lemma 2 in \cite{golowich2018size}) to get
\begin{align*}
    &\mathbb{E}_{\sigma}\left[
    \sup\limits_{h \in \mathcal{H}}
    \norm{
    \frac{1}{N}
    \sum\limits_{i=1}^{N}\sigma_i \rho(h(\x_i))}
    _{\infty}
    \right]
     \leq
    4 \mathbb{E}_{\sigma}\left[
    \sup\limits_{h \in \mathcal{H}}
    \norm{
    \frac{1}{N}
    \sum\limits_{i=1}^{N}\sigma_i h(\x_i)}
    _{\infty}
    \right],
\end{align*}
where we used that $\rho(x) = ReLU(x)$ is 1-Lipschitz and the same logic for the
alternative definition of the Rademacher complexity as in the proof of Lemma
\ref{lem:rc}, which results in a constant factor of 2. The constant 4 is then
obtained by the additional constant factor 2 from Talagrand's lemma. The rest of
proof is identical to the sigmoid case.

The restriction of an MLP to the ball $B_{\ell_T^{\infty}(\mathbb{R}^{n_u})}(r)$  maps to the ball $B_{\ell_T^{\infty}(\mathbb{R}^{n_v})}(K_W r + K_\bfb)$, because the elementwise ReLU does
not increase the infinity norm, hence we can apply Lemma \ref{lem:affine} and Remark \ref{rem:affine}.
Again, for the deep model the result is straightforward from Lemma
\ref{lem:composition} along with Lemma \ref{lem:affine}, Remark \ref{rem:affine}.

\textbf{GLU.}
For the ease of notation,
assume that $K_{GLU,i}=K_{GLU}$ and
let $\mathcal{W}=\{ W \in \mathbb{R}^{n_y \times n_u} \mid \|W\|_{\infty,\infty}< K_{GLU}\}$
and let
$\mathcal{F}_{\text{\normalfont GLU}} = \{ f_{GLU} \mbox{ as in \eqref{def:glu:eq1} } \mid   W \in \mathcal{W}
            \}$.
As $\mathcal{F}_i \subseteq \mathcal{F}_{\text{\normalfont{GLU}}}$, it is enough to prove the claim of the lemma for $\mathcal{F}_{\text{\normalfont GLU}}$.

    First of all, we show that the function $h: (\mathbb{R}^2, \norm{\cdot}_{2})
    \to (\mathbb{R}, | \cdot |)$ defined as $h(\x) = x_1 \cdot \sigma(x_2)$ is
    $\sqrt{2}(K+1)$-Lipschitz on a bounded domain, where $|x_i| \leq
    K$ for all $\x \in \mathbb{R}^2$ we consider. We will later specify the
    value of $K$ in relation to Assumption \ref{ass:all}. By the sigmoid being
    1-Lipschitz, we have
    \begin{align*}
        &| h(\x) - h(\y) | = | x_1 \sigma(x_2) - y_1 \sigma(x_2)
         + y_1 \sigma(x_2) - y_1 \sigma(y_2)| \leq \\
         &| (x_1 - y_1) \sigma(x_2)| + |y_1 (\sigma(x_2) - \sigma(y_2)) |
         \leq |x_1 - y_1| + |y_1| |x_2 - y_2| \\
         & \leq \sqrt{2}(K+1) \norm{\x - \y}_{2} 
    \end{align*}

    Second, we recall Corollary 4 in \cite{maurer2016vector}.
    \begin{theorem}[\cite{maurer2016vector}]
        \label{thm:maurer}
       Let $\mathcal{X}$ be any set, $(\x_1, \ldots , \x_N) \in \mathcal{X}^N$,
       let $\widetilde{\mathcal{F}}$ be a set of functions $f: \mathcal{X} \to
       \ell_T^2(\mathbb{R}^m)$ and let $h: \ell_T^2(\mathbb{R}^m) \to \mathbb{R}$ be
       an $L$-Lipschitz function. Under $f_k$ denoting the $k$-th component
       function of $f$ and $\sigma_{ik}$ being a doubly indexed Rademacher
       variable, we have
       \begin{align*}
        \mathbb{E}_\sigma \left[
            \sup\limits_{f \in \widetilde{\mathcal{F}}}\sum\limits_{i=1}^N
             \sigma_i h(f(\x_i))
             \right] \leq \sqrt{2} L \mathbb{E}_\sigma \left[
                \sup\limits_{f \in \widetilde{\mathcal{F}}}\sum\limits_{i=1}^N
                \sum\limits_{k=1}^m
                \sigma_{ik} f_k(\x_i)
             \right].
       \end{align*}
    \end{theorem}

    We wish to apply Theorem \ref{thm:maurer} to GLU layers.
    For any $Z \subseteq \ell_T^{\infty}(\mathbb{R}^{n_u})$, by letting
    $GLU_W(\z) = f_{GLU}(\z)$ we have
    \begin{align*}
        &\mathbb{E}_{\sigma}\left[
        \sup\limits_{W \in \mathcal{W}}
        \sup\limits_{\{\z_i \}_{i=1}^N \in Z}
        \norm{
        \frac{1}{N}
        \sum\limits_{i=1}^{N}\sigma_i GLU_W(\z_i)}
        _{\ell_T^{\infty}(\mathbb{R}^{n_u})}
        \right] \\
        &= \mathbb{E}_{\sigma}\left[
        \sup\limits_{W \in \mathcal{W}}
        \sup\limits_{\{\z_i \}_{i=1}^N \in Z}
        \sup\limits_{1 \leq k \leq T}
        \sup\limits_{1 \leq j \leq n_u}
        \left|\frac{1}{N}
        \sum\limits_{i=1}^{N}\sigma_i GLU_W^{(j)}(\z_i)[k]\right|
        \right]. \\
    \end{align*}
    Now this is an alternative version of the Rademacher complexity, where we
    take the absolute value of the Rademacher average. In order to apply Theorem
    \ref{thm:maurer}, we reduce the problem to the usual Rademacher complexity.
    In turn, we can apply the last chain of inequalities in the proof of
    Proposition 6.2 in \cite{hajek2019ece}. Concretely,
    by denoting $\mathbf{O}=\{\mathbf{0}\}_{i=1}^N$ and
    noticing that $GLU_{W}(0)=0$, we have
    \begin{align*}
        &\mathbb{E}_{\sigma}\left[
        \sup\limits_{W \in \mathcal{W}}
        \sup\limits_{\{\z_i \}_{i=1}^N \in Z}
        \sup\limits_{1 \leq k \leq T}
        \sup\limits_{1 \leq j \leq n_u}
        \left|\frac{1}{N}
        \sum\limits_{i=1}^{N}\sigma_i GLU_W^{(j)}(\z_i)[k]\right|
        \right] \\
        & \leq
        2\mathbb{E}_{\sigma}\left[
        \sup\limits_{W \in \mathcal{W}}
        \sup\limits_{\{\z_i \}_{i=1}^N \in Z \cup \{\mathbf{O}\}}
        \sup\limits_{1 \leq k \leq T}
        \sup\limits_{1 \leq j \leq n_u}
        \frac{1}{N}
        \sum\limits_{i=1}^{N}\sigma_i GLU_W^{(j)}(\z_i)[k]
        \right]. \\
    \end{align*}
    Let $\x_i=i$, $i=1,\ldots,N$ and let $\mathcal{H} = \{ f_{W, \underline{z},
    k, j} \mid (W, \underline{z}, k, j) \in \mathcal{W} \times (Z \cup \{0\})
    \times [T] \times [n_u]\}$ such that
    $f_{W,\underline{z},k,j}(\x_i)=\begin{bmatrix} GELU(\z_i[k])^{(j)} &
    (W(GELU(\z_i[k])))^{(j)} \end{bmatrix}^T$ for $\underline{z}=\{\z_i
    \}_{i=1}^N \in Z$.
    Since $Z \subseteq (B_{\ell_T^{\infty}(\mathbb{R}^{n_{y}})}(r)^N$, it follows
    that for all $\{\z_i
    \}_{i=1}^N \in Z$ and for all $k \in \mathbb{N}$,
    the inequality $\|\z_i[k]\|_{\infty}
    \le r$ holds. Hence, $|GELU(\z_i[k])^{(j)}| < r$, leading to
    $|W(GELU(\z_i[k]))^{(j)}| < \sup\limits_{W \in \mathcal{W}}
    \|W\|_{\infty,\infty} \cdot r$. In particular,
    $GLU^{(j)}_{W}(\z_i)[k]=h(f_{W,\underline{z},k,j}(\x_i))=
    h|_{B}(f_{W,\underline{z},k,j}(\x_i))$, where $h|_B$ is the restriction of
    $h$ to $B=\{ x \in \mathbb{R}^2 \mid \|x\|_{\infty} < K\}$,
    with
    $K=\max\{r,
    \sup\limits_{W \in \mathcal{W}} \|W\|_{\infty,\infty} \cdot r\}$.
    In particular, $h|_B$
    is $\sqrt{2}(K+1)$-Lipschitz.

     We are ready to apply Theorem \ref{thm:maurer}, together with the GLU
    definition and its $\sqrt{2}(K+1)$-Lipschitzness, we have
    \begin{align*}
        &2\mathbb{E}_{\sigma}\left[
        \sup\limits_{W \in \mathcal{W}}
        \sup\limits_{\{\z_i \}_{i=1}^N \in Z \cup \{\mathbf{O}\}}
        \sup\limits_{1 \leq k \leq T}
        \sup\limits_{1 \leq j \leq n_u}
        \frac{1}{N}
        \sum\limits_{i=1}^{N}\sigma_i GLU_W^{(j)}(\z_i)[k]
        \right]\\
        & = 2\mathbb{E}_{\sigma}\left[ \sup\limits_{f \in \mathcal{H}}
        \frac{1}{N}
        \sum\limits_{i=1}^{N}\sigma_i h(f(\x_i))
        \right]
        \leq
        4 (K+1)\underbrace{\mathbb{E}_{\sigma}\left[
        \sup\limits_{f \in \mathcal{H}}
        \frac{1}{N}
        \sum\limits_{i=1}^{N}\sigma_i GELU(\z_i[k])^{(j)}
        \right]}_{A}\\
        & +
        4(K+1) \underbrace{\mathbb{E}_{\sigma}\left[
        \sup\limits_{f \in \mathcal{H}}
        \frac{1}{N}
        \sum\limits_{i=1}^{N}\sigma_i W(GELU(\z_i))^{(j)}[k]
        \right]}_{B}
    \end{align*}
    Due to the definition of GELU, its 2-Lipschitzness \cite{qi2023lipsformer}
    and Theorem 4.12 from \cite{ledoux1991probability} we have
    \begin{align*}
        &A = \mathbb{E}_{\sigma}\left[
        \sup\limits_{\{\z_i \}_{i=1}^N \in Z \cup
         \{\mathbf{O}\}}
        \norm{\frac{1}{N}
        \sum\limits_{i=1}^{N}\sigma_i GELU(\z_i)}
        _{\ell_T^{\infty}(\mathbb{R}^{n_u})}
        \right] =  \\
        &
        \leq 4 \mathbb{E}_{\sigma}\left[
        \sup\limits_{\{\z_i \}_{i=1}^N \in  Z \cup
         \{\mathbf{O}\}}
        \norm{\frac{1}{N}
        \sum\limits_{i=1}^{N}\sigma_i \z_i}
        _{\ell_T^{\infty}(\mathbb{R}^{n_u})}
        \right] =
        4 \mathbb{E}_{\sigma}\left[
        \sup\limits_{\{\z_i \}_{i=1}^N \in Z}
        \norm{\frac{1}{N}
        \sum\limits_{i=1}^{N}\sigma_i \z_i}
        _{\ell_T^{\infty}(\mathbb{R}^{n_u})}
        \right]
    \end{align*}
    and
    \begin{align*}
        &B = \mathbb{E}_{\sigma}\left[
        \sup\limits_{W \in \mathcal{W}}
        \sup\limits_{\{\z_i \}_{i=1}^N \in \{\mathbf{O}\}}
        \norm{\frac{1}{N}
        \sum\limits_{i=1}^{N}\sigma_i W(GELU(\z_i))}
        _{\ell_T^{\infty}(\mathbb{R}^{n_u})}
        \right] \\
        & \leq  \sup\limits_{W \in \mathcal{W}}
         \norm{W}_{\infty}
        \mathbb{E}_{\sigma}\left[
        \sup\limits_{\{\z_i \}_{i=1}^N \in Z  \{\mathbf{O}\}}
        \norm{\frac{1}{N}
        \sum\limits_{i=1}^{N}\sigma_i GELU(\z_i)}
        _{\ell_T^{\infty}(\mathbb{R}^{n_u})}
        \right]\\
        & \leq 4 \sup\limits_{W \in \mathcal{W}}
         \norm{W}_{\infty} \mathbb{E}_{\sigma}\left[
        \sup\limits_{\{\z_i \}_{i=1}^N \in Z}
        \norm{\frac{1}{N}
        \sum\limits_{i=1}^{N}\sigma_i \z_i}
        _{\ell_T^{\infty}(\mathbb{R}^{n_u})}
        \right]
    \end{align*}
    Here we used the linearity of $W$ and the exact same calculation as in the
    proof of Lemma \ref{lem:affine}.

    By combining the inequalities above, it follows that
    \begin{align*}
        & \mathbb{E}_{\sigma}\left[
        \sup\limits_{W \in \mathcal{W}}
        \sup\limits_{\{\z_i \}_{i=1}^N \in Z}
        \sup\limits_{1 \leq k \leq T}
        \sup\limits_{1 \leq j \leq n_u}
        \left|\frac{1}{N}
        \sum\limits_{i=1}^{N}\sigma_i GLU_W^{(j)}(\z_i)[k]\right|
        \right] \leq \\
      & 16(K+1)\left(\sup\limits_{W \in \mathcal{W}}
       \norm{W}_{\infty, \infty} + 1\right)
        \mathbb{E}_{\sigma}\left[
        \sup\limits_{\{\z_i \}_{i=1}^N \in Z}
        \norm{\frac{1}{N}
        \sum\limits_{i=1}^{N}\sigma_i \z_i}
        _{\ell_T^{\infty}(\mathbb{R}^{n_u})}
        \right]
    \end{align*}

    Substituting the value of $K$ gives the result.

\textbf{SSM block.}
 By Lemma \ref{lem:composition} we have that the composition
     of the SSM layer $\mathcal{E}$ and a non-linear layer $\mathcal{F}_i$ which  is $(\mu_i(r),c_i(r))$-RC
     for $i > 2$ and it is
     $(\mu_i(r)K_2, c_i(r))$-RC for $i=1$.
     A residual SSM
     block is then $(\mu_i(r)K_{l(i)} + \alpha_i, c(r))$-RC,  where $K_{l(1)}=K_{2}$
     $K_{l(2)}=K_1$,
     because
     \begin{align*}
         &\mathbb{E}_\sigma \left[
             \sup\limits_{g \circ \mathcal{S}_{\Sigma}}
             \sup\limits_{\{\z_j\}_{j=1}^N \in Z}
             \norm{\frac{1}{N} \sum\limits_{j=1}^N
              \sigma_j (g(\mathcal{S}_\Sigma(\z_j)) + \alpha \z_j)}
             _{\ell_T^{\infty}(\mathbb{R}^{n_u})} \right] \leq \\
        &\mathbb{E}_\sigma \left[
             \sup\limits_{g \circ \mathcal{S}_{\Sigma}}
             \sup\limits_{\{\z_j\}_{j=1}^N \in Z}
             \norm{\frac{1}{N} \sum\limits_{j=1}^N
             \sigma_j g(\mathcal{S}_\Sigma(\z_j))}
             _{\ell_T^{\infty}(\mathbb{R}^{n_u})}
         \right] + \alpha
         \mathbb{E}_\sigma \left[
             \sup\limits_{\{\z_j\}_{j=1}^N \in Z}
             \norm{\frac{1}{N} \sum\limits_{j=1}^N \sigma_j \z_j}
             _{\ell_T^{\infty}(\mathbb{R}^{n_u})}
         \right]\\
         &\leq
         (\mu_i(r)K_{l(i)} + \alpha_i)\mathbb{E}_\sigma \left[
             \sup\limits_{\{\z_j\}_{j=1}^N \in Z}
             \norm{\frac{1}{N} \sum\limits_{j=1}^N \sigma_j \z_j}
             _{\ell_T^{\infty}(\mathbb{R}^{n_u})} \right]
              + \frac{c_i(r)}{\sqrt{N}}
     \end{align*}
\end{proof}
\begin{proof}[Proof of Lemma \ref{inp:bound}]
By definition
\begin{align*}
    &\norm{\frac{1}{N}\sum\limits_{i=1}^N \sigma_i \uu_i}
_{\ell_T^{2}(\mathbb{R}^{\nin})}
= \sqrt{ \sum\limits_{k=1}^{T} \norm{\frac{1}{N}
\sum\limits_{i=1}^N \sigma_i \uu_i[k]}^2_2} \\
&= \sqrt{ \sum\limits_{k=1}^{T}
\left\langle \frac{1}{N} \sum\limits_{i=1}^N \sigma_i \uu_i[k],
 \frac{1}{N} \sum\limits_{j=1}^N \sigma_j \uu_j[k]
  \right\rangle_{\mathbb{R}^{\nin}}} \\
&= \sqrt{ \sum\limits_{k=1}^{T}
  \frac{1}{N^2}\sum\limits_{i=1}^N \sum\limits_{j=1}^N
 \sigma_i \sigma_j \left\langle  \uu_i[k], \uu_j[k]
  \right\rangle_{\mathbb{R}^{\nin}}} \\
\end{align*}
where
$\left\langle \cdot,\cdot \right\rangle_{\mathbb{R}^{\nin}}$
denotes the standard scalar product in $\mathbb{R}^{\nin}$.
Therefore
\begin{align*}
    &\mathbb{E}_\sigma \left[
    \norm{\frac{1}{N}\sum\limits_{i=1}^N \sigma_i \uu_i}
    _{\ell_T^{2}(\mathbb{R}^{\nin})}
    \right] =
    \mathbb{E}_\sigma \left[
 \sqrt{ \sum\limits_{k=1}^{T}
  \frac{1}{N^2}\sum\limits_{i=1}^N \sum\limits_{j=1}^N
 \sigma_i \sigma_j \left\langle  \uu_i[k], \uu_j[k]
  \right\rangle_{\mathbb{R}^{\nin}}}
    \right] \\
    & \leq
  \sqrt{\mathbb{E}_\sigma \left[
  \sum\limits_{k=1}^{T}
  \frac{1}{N^2}\sum\limits_{i=1}^N \sum\limits_{j=1}^N
 \sigma_i \sigma_j \left\langle  \uu_i[k], \uu_j[k]
  \right\rangle_{\mathbb{R}^{\nin}}\right]}  \\
    & =
  \sqrt{
  \sum\limits_{k=1}^{T}
  \frac{1}{N^2}\sum\limits_{i=1}^N \sum\limits_{j=1}^N
 \mathbb{E}_\sigma \left[ \sigma_i \sigma_j\right]
 \left\langle  \uu_i[k], \uu_j[k]
  \right\rangle_{\mathbb{R}^{\nin}}}  \\
    & =
  \sqrt{
  \sum\limits_{k=1}^{T}
  \frac{1}{N^2}\sum\limits_{i=1}^N
 \mathbb{E}_\sigma \left[ \sigma_i^2\right]
 \left\langle  \uu_i[k], \uu_i[k]
  \right\rangle_{\mathbb{R}^{\nin}}}  \\
    & =
  \sqrt{
 \frac{1}{N^2} \sum\limits_{i=1}^{N}
  \sum\limits_{k=1}^{T}
 \norm{\uu_i[k]}_2^2}
   =
  \sqrt{
 \frac{1}{N^2} \sum\limits_{i=1}^{N}
  \norm{\uu_i}^2_{\ell_T^{2}(\mathbb{R}^{\nin})}}
  \leq
  \sqrt{
 \frac{1}{N^2} N K_\uu^2} \leq \frac{K_\uu}{\sqrt{N}}
\end{align*}
\end{proof}
\begin{proof}[Proof of Theorem \ref{thm:maingeneral}]
From Lemma \ref{lem:rc}
it follows that all maps constituting
a model $f \in \mathcal{F}$ come from families of
maps which are $(\mu,c)$-RC for suitable
constants $\mu,c$, and map any ball of radius $r$
to a ball of radius $\hat{r}(r)$.
Let us consider the deep SSM model \eqref{def:dtdeepssm:eq1}, which is a composite of mappings as
\begin{align*}
   & B_{\ell_T^{2}(\mathbb{R}^{\nin})}(K_{\uu}) \xrightarrow{\text{Encoder}}
   B_{\ell_T^{2}(\mathbb{R}^{n_u})}(K_\uu K_{\enc}) \xrightarrow{\bb_1}
    B_{\ell_T^{\infty}(\mathbb{R}^{n_u})}(r_1) \xrightarrow{\bb_2}
    \ldots \xrightarrow{\bb_L} \\
    & B_{\ell_T^{\infty}(\mathbb{R}^{n_u})}(r_{L})
    \xrightarrow{\text{Pooling}} B_{(\mathbb{R}^{n_u}, \norm{\cdot}_{\infty})}(r_L)
    \xrightarrow{\text{Decoder}} B_{(\mathbb{R}, | \cdot |)}(K_{\dec}r_L),
\end{align*}
where the constants
$r_i$, $i \in [L]$ are as in \eqref{ri_eq}, due to repeated application of  Lemma \ref{lem:rc} and
the expressions in Table \ref{tableRC}.

Note that
the first SSM block is considered as a map
$B_{\ell_T^{2}(\mathbb{R}^{n_u})}(K_{\enc} K_{\uu}) \to
B_{\ell_T^{\infty}(\mathbb{R}^{n_u})}(r_1)$, while the rest of the SSM
layers in the SSM blocks are considered as a map
$B_{\ell_T^{\infty}(\mathbb{R}^{n_u})}(r_i) \to
B_{\ell_T^{\infty}(\mathbb{R}^{n_u})}(r_{i+1})$. This
is needed, because the encoder is constant in time, therefore the
Composition Lemma wouldn't be able to carry the $\ell_T^{2}$ norm of the
input through the chain of estimation along the entire model. This is one
of the key technical points which makes it possible to establish a time
independent bound.

Next, we wish to apply Theorem \ref{thm:pac} to the set of deep SSM models
$\mathcal{F}$. Let us fix a random sample $S = \{ \uu_1,\ldots,\uu_N \}
\subset \left(\ell_T^{2}(\mathbb{R}^{\nin})\right)^N$.
As the loss function is Lipschitz according to Assumption \ref{ass:all},
we have that for any $f \in \mathcal{F}$
\begin{align*}
   |l(f(\uu), y)| \leq 2 L_l \max\{ f(\uu), y\} \leq
    2 L_l \max\{K_\dec r_L, K_y \},
\end{align*}
thus $K_l \leq 2 L_l \max\{K_\dec r_L, K_y \}$. 
Again by the Lipschitzness of the loss and the Contraction Lemma
(Lemma 26.9 in \cite{shalev2014understanding}) we have
\begin{align*}
   R_S(L_0) \leq L_l \cdot R_S(\mathcal{F}).
\end{align*}

It is enough to bound the Rademacher complexity of $\mathcal{F}$ to
conclude the proof.
By applying Lemma \ref{lem:rc} to every layer
of $\mathcal{F}$ and using Lemma \ref{lem:composition}, it follows that
the family $\mathcal{F}|_{X_1}$
of restriction of the elements
$\mathcal{F}$  to $X_1 =
B_{\ell_T^{2}(\mathbb{R}^{\nin})}(K_{\uu})$
is a family of maps from $X_1$ to $X_2 = (\mathbb{R}, |\cdot|)$ which
is $(\mu,c)$-RC, where $\mu,c$
are as in \eqref{muc:const}.
Next, we state a lemma before we finish the proof.
    \begin{lemma}
    \label{lemma_d5}
        Let $\mathcal{F}$ be a set of functions between
        $X_1 = B_{\ell_T^{2}(\mathbb{R}^{\nin})}(K_{\uu})$ and $X_2 =
        (\mathbb{R}, |\cdot|)$ that is $(\mu, c)$-RC. The Rademacher complexity of $\mathcal{F}$
        w.r.t. some dataset $S$ for which Assumption \ref{ass:all} holds, admits
        the following inequality.
        \begin{align*}
            R_{S}(\mathcal{F}) \leq \frac{\mu K_\uu + c}{\sqrt{N}}.
        \end{align*}
    \end{lemma}

\begin{proof}
\begin{align*}
    &R_{S}(\mathcal{F}) = R(\left\{(f(\uu_1),\dots,f(\uu_N))^T
     \mid f \in \mathcal{F}
    \right\}) = \mathbb{E}_{\sigma}\left[\sup\limits_{f \in \mathcal{F}}
    \frac{1}{N}\sum\limits_{i=1}^N \sigma_i f(\uu_i) \right] \\
    & \leq \mathbb{E}_{\sigma}\left[\sup\limits_{f \in \mathcal{F}}
    \left| \frac{1}{N}\sum\limits_{i=1}^N \sigma_i f(\uu_i) \right|\right]
    \leq \mu \mathbb{E}_{\sigma}\left[
        \norm{\frac{1}{N}\sum\limits_{i=1}^N
        \sigma_i \uu_i}
        _{\ell_T^{2}(\mathbb{R}^{\nin})}
        \right] + \frac{c}{\sqrt{N}}
\end{align*}
By applying Lemma \ref{inp:bound}, it follows
\begin{align*}
    R_S(\mathcal{F}) \leq \frac{\mu K_\uu + c}{\sqrt{N}}
\end{align*}
\end{proof}
The Theorem is then a direct corollary of Lemma~\ref{lemma_d5}.
\end{proof}

\section{Numerical example}
\label{num}
Let $\{ \theta_t \}_{1 \leq t \leq T}$ be a standard
normal sample sorted in ascending order and $\varphi_t = 7 \pi
\sqrt{\theta_t}$. Consider the following two classes of time
series, labeled by 0 and 1.
\begin{align*}
    \x_0[t] &= ((2 \varphi_t + \pi) \cos(\varphi_t),(2 \varphi_t + \pi) \sin(\varphi_t))^T + \bm{\varepsilon} \\
    \x_1[t] &= (-(2 \varphi_t + \pi) \cos(\varphi_t),-(2 \varphi_t + \pi) \sin(\varphi_t))^T + \bm{\varepsilon}
\end{align*}
where $1 \leq t \leq T$ for some $T$ and $\bm{\varepsilon}$ are i.i.d standard
norma noise vectors.
The training data is depicted on Figure \ref{fig:spiral}.

We trained a linear SSM on this binary classification problem with a training set size of $N$, for various values of $N$. The resulting parameter vector therefore depends on $N$. We applied ADAM on the binary cross-entropy loss combined with applying a sigmoid activation to the scores outputted by the model. Theorem \ref{thm:maingeneral} states that with high probability, we can upper bound the true loss with the sum of the empirical loss and a bounding term. In this case, the true loss is estimated by taking the loss of the model on a very large set of samples (in this case, 150 000). It can be seen on Figure \ref{fig:pac} that the bound holds and in this case it is non-vacuous, i.e. there exists a value of the true loss - highlighted by the red broken line - which is greater, than the value of the estimation at a different value of $N$.


Moreover, it can also be seen on Figure \ref{fig:learning} that during the learning process, the numerical value of the bounding term, and therefore the estimation of the true loss, correlates with the true performance of the model. Once the learning algorithm passed its optimum where the accuracy is close to 1, and starts to exhibit overfitting, the bounding term and the estimation start to rapidly grow, while the value of the loss stays consistently low.

\begin{figure}[H]
    \centering
    \includegraphics[width=.9\textwidth]{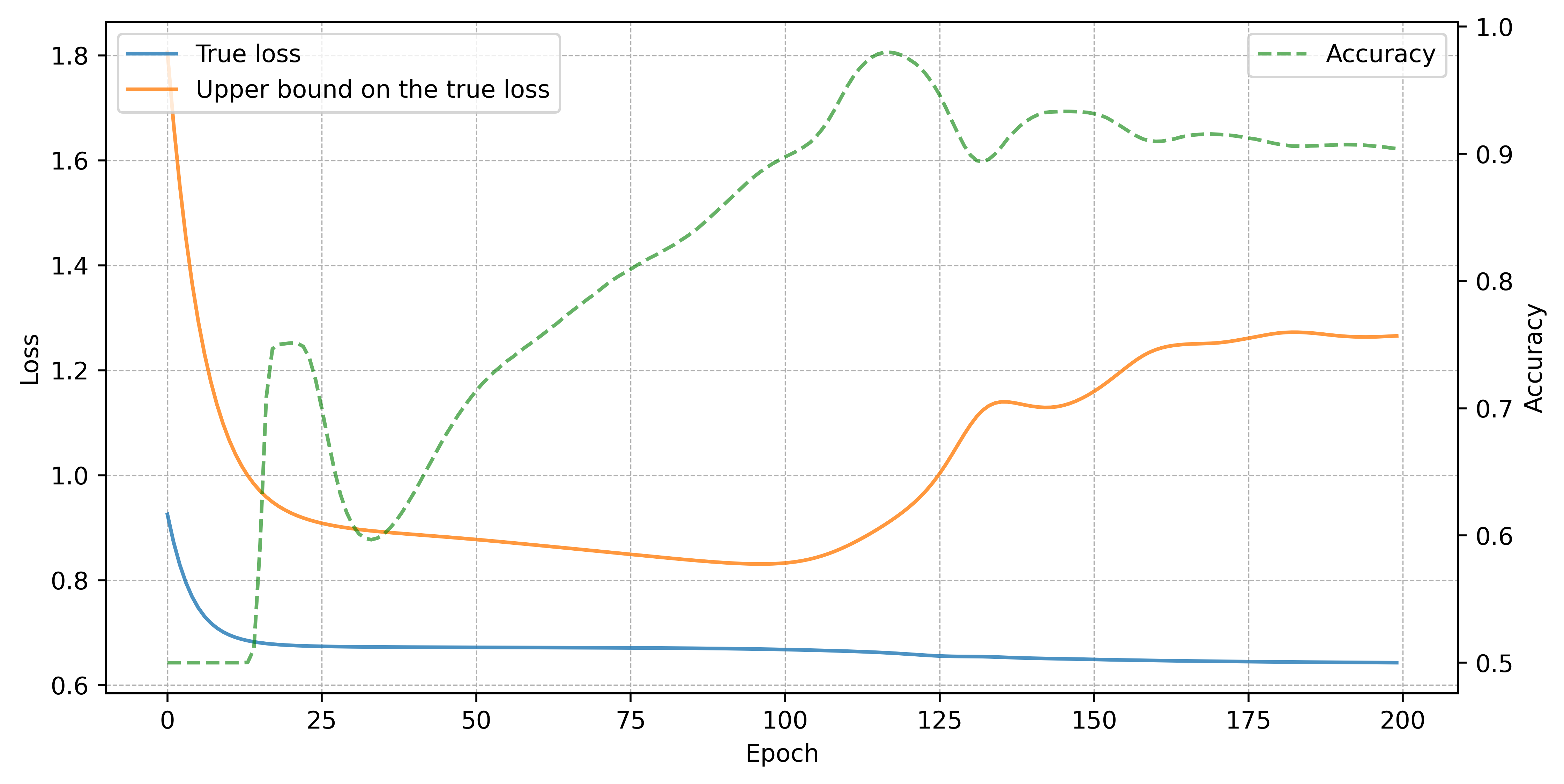}
    \caption{Behavior of the bound on the true loss during learning.}
    \label{fig:learning}
\end{figure}

\end{document}